%% file: main.tex
\documentclass{article}

    \usepackage[final]{neurips_2022}

\usepackage{header}

\usepackage{wrapfig}

\newif\ifappendix
\appendixfalse

\title{Reinforcement Learning with a Terminator}

\author{%
   Guy Tennenholtz \thanks{Technion, Israel institute of technology}~~\thanks{Nvidia Research, Israel} \\
   \texttt{guytenn@gmail.com} 
   \And
   Nadav Merlis \footnotemark[1]
   \And
   Lior Shani \footnotemark[1]
   \And
   Shie Mannor \footnotemark[1]~~\footnotemark[2]
   \AND
   Uri Shalit \footnotemark[1]
   \And
   Gal Chechik \footnotemark[2]~~\thanks{Bar Ilan University, Israel}
   \And
   Assaf Hallak \footnotemark[2]
   \And
   Gal Dalal \footnotemark[2]
}

\begin{document}

\doparttoc
\faketableofcontents

\maketitle






\begin{abstract}
    We present the problem of reinforcement learning with exogenous termination. We define the Termination Markov Decision Process (TerMDP), an extension of the MDP framework, in which episodes may be interrupted by an external non-Markovian observer. This formulation accounts for numerous real-world situations, such as a human interrupting an autonomous driving agent for reasons of discomfort. We learn the parameters of the TerMDP and leverage the structure of the estimation problem to provide state-wise confidence bounds. We use these to construct a provably-efficient algorithm, which accounts for termination, and bound its regret. Motivated by our theoretical analysis, we design and implement a scalable approach, which combines optimism (w.r.t. termination) and a dynamic discount factor, incorporating the termination probability. We deploy our method on high-dimensional driving and MinAtar benchmarks. Additionally, we test our approach on human data in a driving setting. Our results demonstrate fast convergence and significant improvement over various baseline approaches.
\end{abstract}

\section{Introduction}

The field of reinforcement learning (RL) involves an agent interacting with an environment, maximizing a cumulative reward \citep{puterman2014markov}. As RL becomes more instrumental in real-world applications \citep{lazic2018data,kiran2021deep,mandhane2022muzero}, exogenous inputs beyond the prespecified reward pose a new challenge. Particularly, an external authority (e.g., a human operator) may decide to terminate the agent's operation when it detects undesirable behavior. In this work, we generalize the basic RL framework to accommodate such external feedback.

We propose a generalization of the standard Markov Decision Process (MDP), in which external termination can occur due to a non-Markovian observer. When terminated, the agent stops interacting with the environment and cannot collect additional rewards. This setup describes various real-world scenarios, including: passengers in autonomous vehicles \citep{le2015autonomous,zhu2020safe}, users in recommender systems \citep{wang2009recommender}, employees terminating their contracts (churn management) \citep{sisodia2017evaluation}, and operators in factories; particularly, datacenter cooling systems, or other safety-critical systems, which require constant monitoring and rare, though critical, human takeovers \citep{modares2015optimized}. In these tasks, human preferences, incentives, and constraints play a central role, and designing a reward function to capture them may be highly complex. Instead, we propose to let the agent itself learn these latent human utilities by leveraging the termination events.

We introduce the Termination Markov Decision Process (TerMDP), depicted in \Cref{fig: termination diagram}. We consider a terminator, observing the agent, which aggregates penalties w.r.t. a predetermined, state-action-dependent, yet \emph{unknown}, cost function. As the agent progresses, unfavorable states accumulate costs that gradually increase the terminator's inclination to stop the agent and end the current episode. Receiving merely the sparse termination signals, the agent must learn to behave in the environment, adhering to the terminator's preferences while maximizing reward. 

Our contributions are as follows. \textbf{(1)} We introduce a novel history-dependent termination model, a natural extension of the MDP framework which incorporates non-trivial termination (\Cref{sec: perliminaries}). \textbf{(2)} We learn the unknown costs from the implicit termination feedback (\Cref{sec: theory}), and provide local guarantees w.r.t. every visited state. We leverage our results to construct a tractable algorithm and provide regret guarantees. \textbf{(3)} Building upon our theoretical results, we devise a practical approach that combines optimism with a cost-dependent discount factor, which we test on MinAtar \citep{young19minatar} and a new driving benchmark. \textbf{(4)} We demonstrate the efficiency of our method on these benchmarks as well as on human-collected termination data (\Cref{sec: experiments}). Our results show significant improvement over other candidate solutions, which involve direct termination penalties and history-dependent approaches. We also introduce a new task for RL -- a driving simulation game which can be easily deployed on mobile phones, consoles, and PC \footnote{Code for Backseat Driver and our method, TermPG, can be found at \href{https://github.com/guytenn/Terminator}{https://github.com/guytenn/Terminator}.}.

\begin{figure}[t!]
\centering
\includegraphics[width=0.8\linewidth]{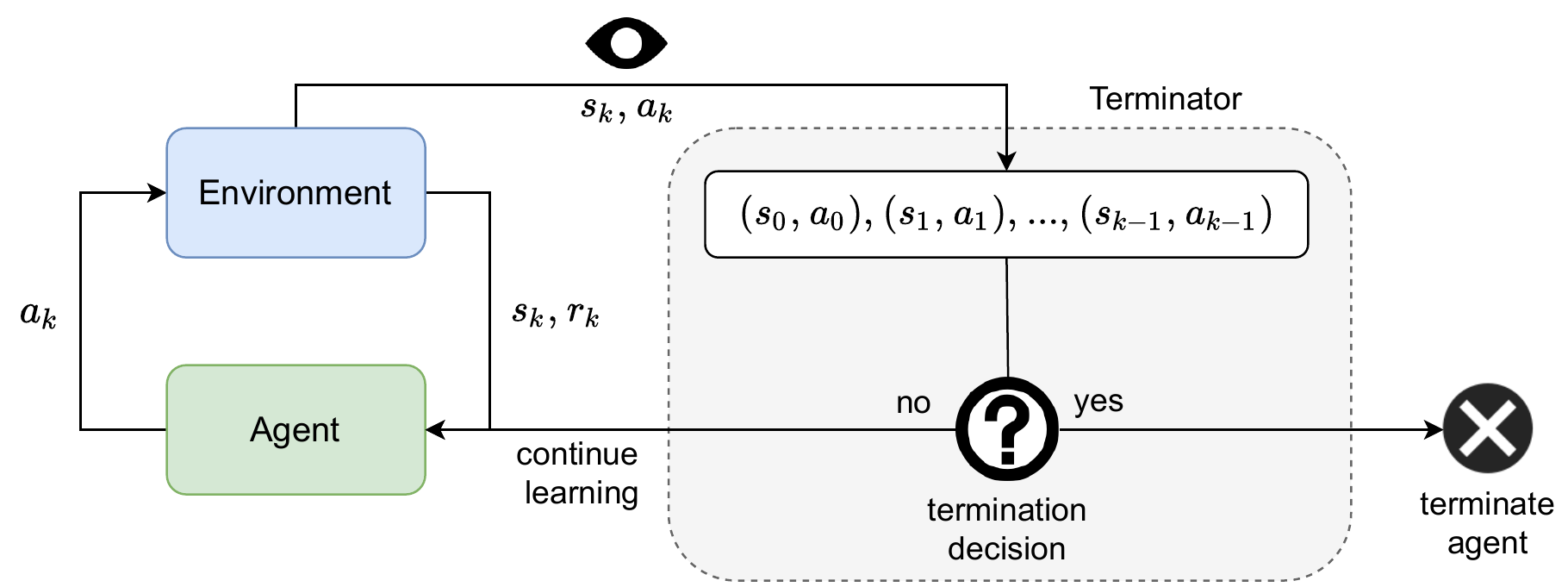}
\caption{\small A block diagram of the TerMDP framework. An agent interacts with an environment while an exogenous observer (i.e., terminator) can choose to terminate the agent based on previous interactions. If the agent is terminated, it transitions to a sink state where a reward of $0$ is given until the end of the episode.}
\label{fig: termination diagram}
\end{figure}

\section{Termination Markov Decision Process}
\label{sec: perliminaries}

We begin by presenting the termination framework and the notation used throughout the paper. Informally, we model the termination problem using a logistic model of past ``bad behaviors". We use an unobserved state-action-dependent cost function to capture these external preferences.
As the overall cost increases throughout time, so does the probability of termination.

For a positive integer $n$, we denote $[n] = \brk[c]*{1, \hdots, n}$. We define the Termination Markov Decision Process (TerMDP) by the tuple $\terMDP=(\sset,\aset, P,R,H,c)$, where $\sset$ and $\aset$ are state and action spaces with cardinality $S$ and $A$, respectively, and $H\in\N$ is the maximal horizon. We consider the following protocol, which proceeds in discrete episodes $k=1, 2, \hdots, K$. At the beginning of each episode~$k$, an agent is initialized at state $s_1^k \in \sset$. At every time step $h$ of episode~$k$, the agent is at state $s_h^k \in \sset$, takes an action $a_h^k \in \aset$ and receives a random reward $R_h^k\in[0,1]$ generated from a fixed distribution with mean $r_h(s_h^k,a_h^k)$. A terminator overseeing the agent utilizes a cost function $c: \brk[s]{H}\times \s \times \A \mapsto \R$ that is unobserved and \emph{unknown to the agent}. At time step $h$, the episode terminates with probability 
\begin{align*}
\rho_h^k(c) = \rho\brk*{\sum_{t=1}^h c_t(s_t^k,a_t^k) - b}, 
\end{align*}
where $\rho(x)=\brk*{1+\exp(-x)}^{-1}$ is the logistic function and $b \in \R$ is a bias term which determines the termination probability when no costs are aggregated. Upon termination, the agent transitions to a terminal state $\terminalstate$ which yields no reward, i.e., $r_h(\terminalstate,a)=0$ for all $h\in\brk[s]*{H}, a\in\aset$. If no termination occurs, the agent transitions to a next state $s_{h+1}^k$ with probability $P_h(s_{h+1}^k | s_h^k,a_h^k)$. Let $t_k^*=\min\brk[c]*{h:s_h^k=\terminalstate}-1$ be the time step when the $k^{\text{th}}$ episode was terminated. Notice that the termination probability is non-Markovian, as it depends on the entire trajectory history. We also note that, when $c \equiv 0$, the TerMDP reduces to a finite horizon MDP with discount factor $\gamma = \rho\brk*{-b}$. Finally, we note that our model allows for negative costs. Indeed, these may capture satisfactory behavior, diminishing the effect of previous mistakes, and decreasing the probability of termination.

We define a stochastic, history dependent policy $\pi_h(s_h, \tau_{1:h})$ which maps trajectories $\tau_{1:h} = (s_1, a_1, \hdots, s_{h-1}, a_{h-1})$ up to time step $h$ (excluding) and the $h^{\text{th}}$ states $s_{h}$ to probability distributions over $\aset$. Its value is defined by 
${
    V_h^{\pi}(s,\tau) \!=\! \E{\sum_{t=h}^H r_t(s_t, a_t) | s_h=s, \tau_{1:h} = \tau, a_t \sim \pi_t(s_t,\tau_{1:t})}.  
}$
With slight abuse of notation, we denote the value at the initial time step by $V_1^\pi(s)$. An optimal policy $\pi^*$ maximizes the value for all states and histories simultaneously \footnote{Such a policy always exists; we can always augment the state space with the history, which would make the environment Markovian and imply the existence of an optimal history-dependent policy \citep{puterman2014markov}.}; we denote its value function by $V^*$. We measure the performance of an agent by its \emph{regret}; namely, the difference between the cumulative value it achieves and the value of an optimal policy,
\begin{align*}
    \Reg{K} = \sum_{k=1}^K V_1^*(s_1^k) - V_1^{\pi^k}(s_1^k). 
\end{align*}

\textbf{Notations. } We denote the Euclidean norm by $\norm{\cdot}_2$ and the Mahalanobis norm induced by the positive definite matrix $A\succ 0$ by $\norm{x}_A=\sqrt{x^TAx}$. We denote by $n_h^k(s,a)$ the number of times that a state action pair $(s,a)$ was visited at the $h^{\text{th}}$ time step before the $k^{\text{th}}$ episode. Similarly, we denote by $\hat{X}_h^k(s,a)$ the empirical average of a random variable $X$ (e.g., reward and transition kernel) at $(s,a)$ in the $h^{\text{th}}$ time step, based on all samples before the $k^{\text{th}}$ episode. 

We assume there exists a known constant $L$ that bounds the norm of the costs; namely, $\sqrt{\sum_{s, a} \sum_{t=1}^H c^2_t(s_t,a_t)} \leq L$, and denote the set of possible costs by $\C$. We also denote the maximal reciprocal derivative of the logistic function by
$
    \kappa = \max_{h\in\brk[s]{H}}\enspace\max_{\brk[c]*{(s_t,a_t)}_{t=1}^h\in(\sset\times\aset)^h} \brk*{\dot{\rho}\brk*{\sum_{t=1}^hc_t(s_t,a_t) - b}}^{-1}.
$
This factor will be evident in our theoretical analysis in the next section, as estimating the costs in regions of saturation of the sigmoid is more difficult when the derivative nears zero. Finally, we use $\mathcal{O}(x)$ to refer to a quantity that depends on $x$ up to a poly-log expression in $S, A, K, H, L, \kappa$ and $\log\brk*{\frac{1}{\delta}}$.

\section{An Optimistic Approach to Overcoming Termination}
\label{sec: theory}

Unlike the standard MDP setup, in the TerMDP model, the agent can potentially be terminated at any time step. Consider the TerMDP model for which the costs are \emph{known}. We can define a Markov policy $\pi_h$ mapping augmented states $\sset \times \R$ to a probability distribution over actions, where here, the state space is augmented by the accumulated costs $\sum_{t=1}^{h-1} c_t(s_t, a_t)$ . There exists a policy, which does not use historical information, besides the accumulated costs, and achieves the value of the optimal history-dependent policy (see \Cref{appendix: known costs}). Therefore, when solving for an optimal policy (e.g., by planning), one can use the current accumulated cost instead of the full trajectory history. 

This suggests a plausible approach for solving the TerMDP -- first learn the cost function, and then solve the state-augmented MDP for which the costs are known. This, in turn, leads to the following question: \textbf{can we learn the costs $c$ from the termination signals?} In what follows, we answer this question affirmatively. We show that by using the termination structure, one can efficiently converge to the true cost function \emph{locally} -- for every state and action. We provide uncertainty estimates for the state-wise costs, which allow us to construct an efficient optimistic algorithm for solving the problem.

\textbf{Learning the Costs. } To learn the costs, we show that the agent can effectively gain information about costs even in time steps where no termination occurs. Recall that at any time step $h\in[H-1]$, the agent acquires a sample from a Bernoulli random variable with parameter $p = \rho_h^k(c) = \rho\brk*{\sum_{t=1}^h c_t(s_t^k,a_t^k) - b}$. Notably, a lack of termination, which occurs with probability $1-\rho_h^k(c)$, is also an informative signal of the unknown costs. We propose to leverage this information by recognizing the costs $c$ as parameters of a probabilistic model, maximizing their likelihood. We use the regularized cross-entropy, defined for some $\lambda > 0$ by
\begin{align}
    \label{eq: cost likelihood}
   \mathcal{L}^k_\lambda(c) 
   = 
   \sum_{k'=1}^{k} \sum_{h=1}^{H-1} \left[ \indicator{h< t_{k'}^*}\log\brk*{1- \rho_h^k(c)} 
   +  
   \indicator{h=t_{k'}^*}\log\brk*{\rho_h^k(c)
    }\right] - \lambda\norm{c}_2^2.
\end{align}

By maximizing the cost likelihood in \Cref{eq: cost likelihood}, global guarantees of the cost can be achieved, similar to previous work on logistic bandits \citep{zhang2016online,abeille2021instance}. Particularly, denoting by $\hat{c}^k \in\arg\max \mathcal{L}^k_\lambda(c)$ the maximum likelihood estimates of the costs, it can be shown that for any history, a global upper bound on $\norm{\hat{c}^k-c}_{\Sigma_k}$ can be obtained, where the history-dependent design matrix $\Sigma_k$ captures the empirical correlations of visitation frequencies (see \Cref{supp: cost concentration} for details). Unfortunately, using $\norm{\hat{c}^k-c}_{\Sigma_k}$ amounts to an intractable algorithm \citep{chatterji2021theory}, and thus to an undesirable result.

Instead, as terminations are sampled on \emph{every time step} (i.e., non-terminations are informative signals as well), we show we can obtain a \emph{local} bound on the cost function $c$. Specifically, we show that the error $\abs{\hat c_h^k(s,a) - c_h(s,a)}$ diminishes with $n_h^k(s,a)$. The following result is a main contribution of our work, and the crux of our regret guarantees later on (see \Cref{supp: cost concentration} for proof).

\begin{theorem}[Local Cost Estimation Confidence Bound]
\label{thm: local cost confidence front}
Let $\hat{c}^k \in\arg\max_{c\in\C} \mathcal{L}^k_\lambda(c)$ be the maximum likelihood estimate of the costs.
Then, for any $\delta>0$, with probability of at least $1-\delta$, for all episodes $k\in [K]$, timesteps $h\in[H-1]$ and state-actions $(s,a)\in\sset\times\aset$, it holds that
\begin{align*}
    \abs{\hat c_h^k(s,a) - c_h(s,a) } 
    \leq  
    \Ob\brk*{
        \brk*{n^k_h(s,a)}^{-0.5}
        \sqrt{\kappa SAHL^{3}} 
        \log\brk*{\frac{1}{\delta}\brk*{1+\frac{kL}{S^2A^2H}}}
    }.
\end{align*}
\end{theorem}
We note the presence of $\kappa$ in our upper bound, a common factor \citep{chatterji2021theory}, which is fundamental to our analysis, capturing the complexity of estimating the costs. Trajectories that saturate the logistic function lead to more difficult credit assignment. Specifically, when the accumulated costs are high, any additional penalty would only marginally change the termination probability, making its estimation harder. A similar argument can be made when the termination probability is low.

We emphasize that in contrast to previous work on global reward feedback in RL \citep{chatterji2021theory,efroni2020reinforcement}, which focused specifically on settings in which information is provided only at the end of an episode, the TerMDP framework provides us with additional information whenever no termination occurs, allowing us to achieve strong, local bounds of the unknown costs. This observation is crucial for the design of a computationally tractable algorithm, as we will see both in theory as well in our experiments later on. 

\begin{algorithm}[t!]
\caption{TermCRL: Termination Confidence Reinforcement Learning}
\label{alg: Termination CRL}
\begin{algorithmic}[1]
\STATE{ \textbf{require:} $\lambda >0$} 
\FOR{$k=1, \hdots, K$}
    \FOR{$(h,s,a) \in [H]\times \sset \times \aset$}
        \STATE $\bar r_h^{k}(s,a) = \hat r_h^{k}(s,a) + b^r_{k}(h,s,a) + b^p_{k}(h,s,a)$ 
        \STATE $\bar c_h^{k}(s,a) = \hat c_h^{k}(s,a) - b^{c}_{k}(h,s,a)$ \hfill {\color{gray}// \Cref{appendix: optimism}}
    \ENDFOR       
    \STATE $\pi^k \gets $ TerMDP-Plan$\brk*{\terMDP \brk1{\sset,\aset,H,\bar{r}^{k},\hat{P}^{k},\bar{c}^{k}}}$ \hfill {\color{gray}// \Cref{appendix: planning}}
    \STATE Rollout a trajectory by acting $\pi^k$
    \STATE $\hat c^{k+1} \in \arg\max_{c\in\C} \mathcal{L}^k_\lambda(c)$ \hfill {\color{gray}// \Cref{eq: cost likelihood}}
    \STATE Update $\hat{P}^{k+1}(s,a), \hat{r}^{k+1}(s,a), n^{k+1}(s,a)$ over rollout trajectory
\ENDFOR
\end{algorithmic}
\end{algorithm}

\subsection{Termination Confidence Reinforcement Learning}

We are now ready to present our method for solving TerMDPs with unknown costs. Our proposed approach, which we call Termination Confidence Reinforcement Learning (TermCRL), is shown in \Cref{alg: Termination CRL}. Leveraging the local convergence guarantees of \Cref{thm: local cost confidence front}, we estimate the costs by maximizing the likelihood in \Cref{eq: cost likelihood}. We compensate for uncertainty in the reward, transitions, and costs by incorporating optimism. We define bonuses for the reward, transition, and cost function by $b_k^r(h,s,a) = \Ob\brk*{\sqrt{\frac{\log\brk*{1/\delta}}{n_h^k(s,a)\vee 1}}}, b_k^p(h,s,a)=\Ob\brk*{\sqrt{\frac{SH^2\log\brk*{1/\delta}}{n_h^k(s,a)\vee 1}}}$, and $b_k^c(h,s,a)=\Ob\brk*{\sqrt{\frac{\kappa SAHL^3}{n_h^k(s,a)\vee 1}}\log\brk*{\frac{1}{\delta}}}$ for some $\delta > 0$ (see \Cref{appendix: optimism} for explicit definitions).

We add the reward and transition bonuses to the estimated reward (line 4), while the optimistic cost bonus is applied directly to the estimated costs (line 5). Then, a planner (line 7) solves the optimistic MDP for which the costs are known and are given by their optimistic counterparts. We refer the reader to \Cref{appendix: planning} for further discussion on planning in TerMDPs. The following theorem provides regret guarantees for \Cref{alg: Termination CRL}. Its proof is given in \Cref{appendix: regret analysis} and relies on \Cref{thm: local cost confidence front} and the analysis of UCRL \citep{auer2008near,efroni2019tight}.

\begin{restatable}{theorem}{MainResult}[Regret of TermCRL]
\label{theorem: main}
With probability at least $1-\delta$, the regret of \Cref{alg: Termination CRL} is
    \begin{align*}
        \Reg{K} \leq \Ob\brk*{\sqrt{\kappa S^2A^2H^{8.5}L^3 K\log^3\brk*{ \frac{SAHK}{\delta}}}}.
    \end{align*}
\end{restatable}
Compared to the standard regret of UCRL \citep{auer2008near}, an additional $\sqrt{\kappa AH^4L^3}$ multiplicative factor is evident in our result, which is due to the convergence rates of the costs in \Cref{thm: local cost confidence front}. Motivated by our theoretical results, in what follows we propose a practical approach, inspired by \Cref{alg: Termination CRL}, which utilizes local cost confidence intervals in a deep RL framework.

\begin{algorithm}[t!]
\caption{TermPG}
\label{alg: Termination PG}
\begin{algorithmic}[1]
\STATE{ \textbf{require:} window $w$, number of ensembles $M$, number of rollouts $N$, number of iterations $K$, policy gradient algorithm \texttt{ALG-PG}}
\STATE{ \textbf{initialize:} $\B_{\text{pos}} \gets \emptyset, \B_{\text{neg}} \gets \emptyset, \pi_{\theta} \gets $ random initialization }
\FOR{$k=1,\hdots,K$}
    \STATE Rollout $N$ trajectories using $\pi_{\theta}$, $\mathcal{R} = \brk[c]*{s^i_1, a^i_1, r^i_1, \hdots, s^i_{t^*_i}, a^i_{t^*_i}, r^i_{t^*_i}}_{i=1}^N$.
    \FOR{$i=1, \hdots, N$}
        \STATE Add $t^*_i - 1$ negative examples $\brk2{s_{\max\brk[c]*{1, l-w+1}}, a_{\max\brk[c]*{1, l-w+1}}, \hdots, s_l, a_l}_{l=1}^{t^*-1}$ to $\B_{\text{neg}}$.
        \STATE Add one positive example $\brk2{s_{\max\brk[c]*{1, t^*-w+1}}, a_{t^*-\max\brk[c]*{1, t^*-w+1}}, \hdots, s_{t^*}, a_{t^*}}$.
    \ENDFOR
    \STATE Train bootstrap ensemble $\brk[c]*{c_{\phi_m}}_{m=1}^M$ using binary cross entropy over data $\B_{\text{neg}}, \B_{\text{pos}}$.
    \STATE Augment states in $\mathcal{R}$ by $s_l^i \gets s_l^i \cup \sum_{j=1}^{\min\brk[c]*{w, l}} \min_{m} c_{\phi_m}(s^i_{l-j}, a^i_{l-j})$.
    \STATE Update policy $\pi_{\theta} \gets \texttt{ALG-PG}\brk*{\mathcal{R}}$ with dynamic discount (see \Cref{sec: discount factor}).
\ENDFOR
\end{algorithmic}
\end{algorithm}

\section{Termination Policy Gradient}
\label{sec: termpg}

Following the theoretical analysis in the previous section, we propose a practical approach for solving TerMDPs. Particularly, in this section, we devise a policy gradient method that accounts for the unknown costs leading to termination. We assume a stationary setup for which the transitions, rewards, costs, and policy are time-homogeneous. Our approach consists of three key elements: learning the costs, leveraging uncertainty estimates over costs, and constructing efficient value estimates through a dynamic cost-dependent discount factor. 

\Cref{alg: Termination PG} describes the Termination Policy Gradient (TermPG) method, which trains an ensemble of cost networks (to estimate the costs and uncertainty) over rollouts in a policy gradient framework. We represent our policy and cost networks using neural networks with parameters $\theta, \brk[c]{\phi_m}_{m=1}^M$. At every iteration, the agent rolls out $N$ trajectories in the environment using a parametric policy,~$\pi_\theta$. The rollouts are split into subtrajectories which are labeled w.r.t. the termination signal, where positive labels are used for examples that end with termination. Particularly, we split the rollouts into ``windows" (i.e., subtrajectories of length $w$), where a rollout of length $t^*$, which ends with termination, is split into $t^*-1$ negative examples $\brk*{s_{\max\brk[c]*{1, l-w+1}}, a_{\max\brk[c]*{1, l-w+1}}, \hdots, s_l, a_l}_{l=1}^{t^*-1}$, and one positive example $\brk*{s_{\max\brk[c]*{1, t^*-w+1}}, a_{t^*-\max\brk[c]*{1, t^*-w+1}}, \hdots, s_{t^*}, a_{t^*}}$. Similarly, a rollout of length $H$ which does not end with termination contains $H$ negative examples. We note that by taking finite windows, we assume the terminator ``forgets" accumulated costs that are not recent - a generalization of the theoretical TerMDP model in \Cref{sec: perliminaries}, for which $w=H$. In \Cref{sec: experiments}, we provide experiments of misspecification of the true underlying window width, where this model assumption does not hold.

\subsection{Learning the Costs} 

Having collected a dataset of positive and negative examples, we train a logistic regression model consisting of an ensemble of $M$ cost networks $\brk[c]*{c_{\phi_m}}_{m=1}^M$, shared across timesteps, as depicted in \Cref{fig: cost diagram}. Specifically, for an example $\brk*{s_{\max\brk[c]*{1, l-w+1}}, a_{\max\brk[c]*{1, l-w+1}}, \hdots, s_l, a_l}$ we estimate the termination probability by $\rho\brk*{\sum_{j=1}^{\min\brk[c]*{w,l}} c_{\phi_m}\brk*{s_{l-j+1}, a_{l-j+1}} - b_m}$, where $\brk[c]*{b_m}_{m=1}^M$ are learnable bias parameters. The parameters are then learned end-to-end using the cross entropy loss. We use the bootstrap method \citep{bickel1981some,chua2018deep} over the ensemble of cost networks. This ensemble is later used in \Cref{alg: Termination PG} to produce optimistic estimates of the costs. Particularly, the agent policy $\pi_\theta$ uses the current state augmented by the optimistic cummulative predicted cost, i.e., $s_l^{\text{aug}} = \brk*{s_l, C_{\text{optimistic}}}$, where $C_{\text{optimistic}} = \sum_{j=1}^{\min\brk[c]*{w, l}} \min_{m} c_{\phi_m}(s_{l-j+1}, a_{l-j+1})$. Finally, the agent is trained with the augmented states using a policy gradient algorithm \texttt{ALG-PG} (e.g., PPO \citep{schulman2017proximal}, IMPALA \citep{espeholt2018impala}).

\begin{figure}[t!]
\includegraphics[width=1\linewidth]{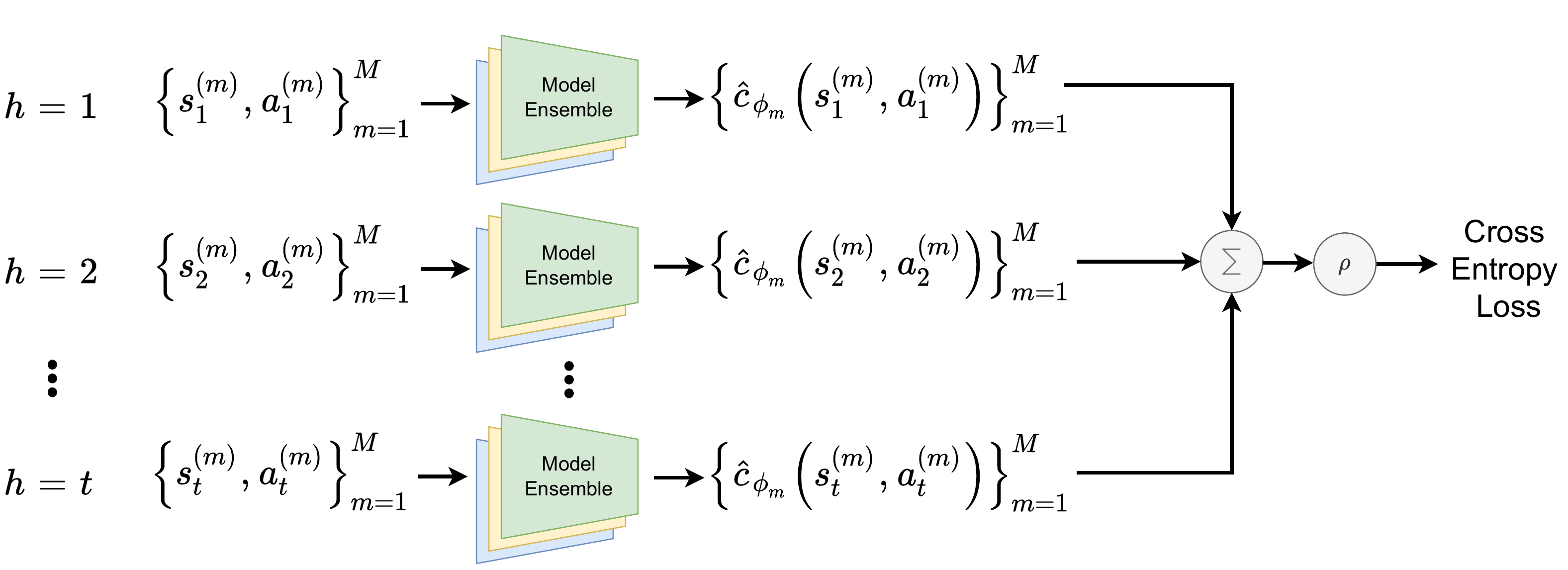}
\caption{\small Block diagram of the cost training procedure. Rollouts are split into subtrajectories, labeled according to whether they end in termination. Given the dataset of labeled subtrajectories, an ensemble of $M$ cost networks is trained end-to-end using cross-entropy with bootstrap samples (all time steps share the same ensemble).} 
\label{fig: cost diagram}
\end{figure}

\subsection{Optimistic Dynamic Discount Factor} 
\label{sec: discount factor}

While augmenting the state with the optimistic accumulated costs is sufficient for obtaining optimality, we propose to further leverage these estimates more explicitly -- noticing that the finite horizon objective we are solving can be cast to a discounted problem. Particularly, it is well known that the discount factor $\gamma \in (0,1)$ can be equivalently formulated as the probability of ``staying alive" (see the discounted MDP framework, \citet{puterman2014markov}). Similarly, by augmenting the state $s$ with the accumulated cost $C_h=\sum_{t=1}^h c(s_t, a_t)$, we view the probability $1-\rho(C_h)$ as a state-dependent discount factor, capturing the probability of an agent in a TerMDP to not be terminated. 

We define a dynamic, cost-dependent discount factor for value estimation. We use the state-action value function $Q(s,a,C)$ over the augmented states, defined for any~$s, a, C$ by
\begin{align*}
    Q^\pi(s,a,C) = \expect*{\pi}{\sum_{t=1}^H \brk*{\prod_{h=1}^t \gamma_h } r(s_t, a_t) | s_1 = s, a_1=a, C_1 = C},
\end{align*}
where ${\gamma_h = 1-\rho\brk*{C + \sum_{i=2}^{h-1} c(s_i, a_i) - b}}$. This yields the Termination Bellman Equations (see \Cref{appendix: termination bellman equations} for derivation)
\begin{align*}
    Q^\pi(s,a,C) 
    =
    r(s,a) + \brk*{1-\rho(C)}\expect*{s' \sim P(\cdot | s,a), a' \sim \pi(s')}{Q^\pi(s', a', C + c(s', a'))}.
\end{align*}
To incorporate uncertainty in the estimated costs, we use the optimistic accumulated costs $C_{\text{optimistic}} = \sum_{j=1}^{\min\brk[c]*{w, l}} \min_{m} c_{\phi_m}(s_{l-j+1}, a_{l-j+1})$. Then, the discount factor becomes ${\gamma\brk*{C_{\text{optimistic}}} = 1-\rho\brk*{C_{\text{optimistic}} - b}}$. Assuming that, w.h.p., optimistic costs are smaller than the true costs, the discount factor decreases as the agent exploits previously visited states. 

The dynamic discount factor allows us to obtain a more accurate value estimator. In particular, we leverage the optimistic cost-dependent discount factor $\gamma\brk*{C_{\text{optimistic}}}$ in our value estimation procedure, using Generalized Advantage Estimation (GAE, \citet{schulman2015high}). As we will show in the next section, using the optimistic discount factor significantly improves overall performance.

\begin{figure}[t!]
\centering
\includegraphics[width=0.4\linewidth]{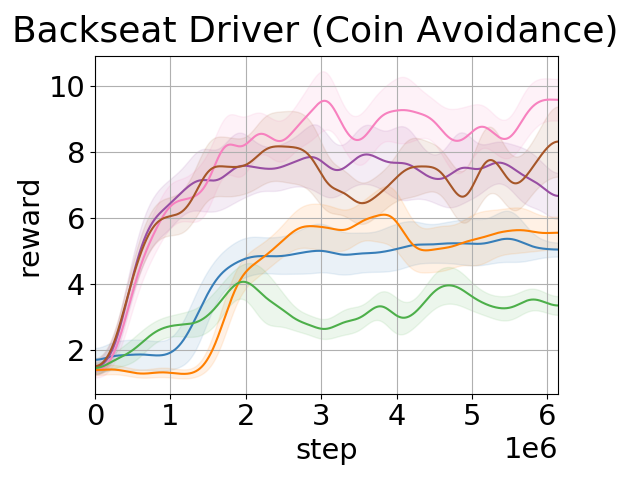}
\includegraphics[width=0.4\linewidth]{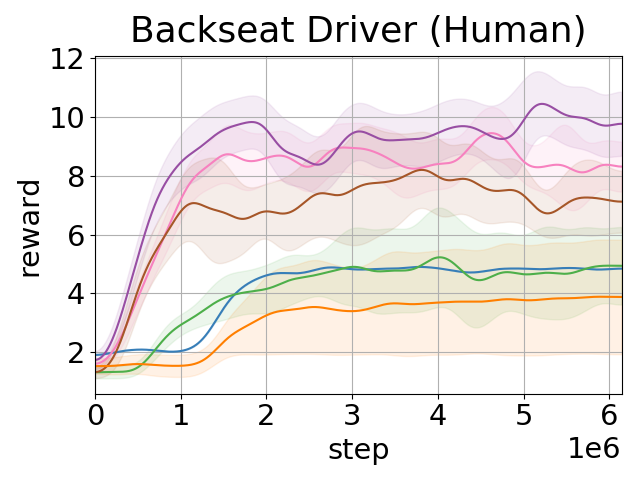}
\includegraphics[width=0.18\linewidth]{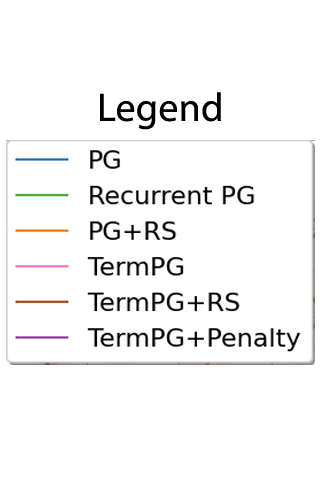}
\caption{\small Mean reward with std. over five seeds of ``Backseat Driver''. Left: coin avoidance; right: human termination. Variants with reward shaping (RS, orange and brown) penalize the agent with a constant value upon termination. The recurrent PG variant (green) uses a history-dependent policy without learning costs. The TermPG+Penalty variant (purple) penalizes the reward at every time step using the estimated costs.} 
\label{fig: backseat driver}
\end{figure}

\section{Experiments}
\label{sec: experiments}

In this section we evaluate the strength of our approach, comparing it to several baselines, including: \textbf{(1) PG (naive):} The standard policy gradient without additional assumptions, which ignores termination. \textbf{(2) Recurrent PG:} The standard policy gradient with a history-dependent recurrent policy (without cost estimation or dynamic discount factor). As the history is a sufficient statistic of the costs, the optimal policy is realizable. \textbf{(3) PG with Reward Shaping (RS):} We penalize the reward upon termination by a constant value, i.e., $r(s,a) - p\indicator{\terminalstate}$, for some $p > 0$. This approach can be applied to any variant of \Cref{alg: Termination PG} or the methods listed above. \textbf{(4) TermPG:} Described in \Cref{alg: Termination PG}. We additionally implemented two variants of TermPG, including: \textbf{(5) TermPG with Reward Shaping:} We penalize the reward with a constant value upon termination. \textbf{(6) TermPG with Cost Penalty:} We penalize the reward at every time step by the optimistic cost estimator, i.e., $r - \alpha C_{\text{optimistic}}$ for some $\alpha > 0$. All TermPG variants used an ensemble of three cost networks, and a dynamic cost-dependent discount factor, as described in \Cref{sec: discount factor}. We report mean and std. of the total reward (without penalties) for all our experiments.

\textbf{Backseat Driver (BDr). } We simulated a driving application, using MLAgents \citep{juliani2018unity}, by developing a new driving benchmark, ``Backseat Driver" (depicted in \Cref{fig: backseat driver ingame}), where we tested both synthetic and human terminations. The game consists of a five lane never-ending road, with randomly instantiating vehicles and coins. The agent can switch lanes and is rewarded for overtaking vehicles. In our experiments, states were represented as top view images containing the position of the agent, nearby cars, and coins with four stacked frames. We used a finite window of length $120$ for termination ($30$ agent decision steps), mimicking a passenger forgetting mistakes of the past.

\textbf{BDr Experiment 1: Coin Avoidance.} In the first experiment of Backseat Driver, coins are considered as objects the driver must avoid. The coins signify unknown preferences of the passenger, which are not explicitly provided to the agent. As the agent collects coins, a penalty is accumulated, and the agent is terminated probabilistically according to the logistic cost model in \Cref{sec: perliminaries}. We emphasize that, while the coins are visible to the agent (i.e., part of the agent's state), the agent only receives feedback from collecting coins through implicit terminations. 

Results for Backseat Driver with coin-avoidance termination are depicted in \Cref{fig: backseat driver}. We compared TermPG (pink) and its two variants (brown, purple) to the PG (blue), recurrent PG (green), and reward shaping (orange) methods described above. Our results demonstrate that TermPG significantly outperforms the history-based and penalty-based baselines. We found TermPG (pink) to perform significantly better, doubling the reward of the best PG variant. All TermPG variants converged quickly to a good solution, suggesting fast convergence of the costs (see \Cref{appendix: additional results}).

\begin{wrapfigure}[12]{R}{0pt}
\centering
\includegraphics[width=0.4\linewidth]{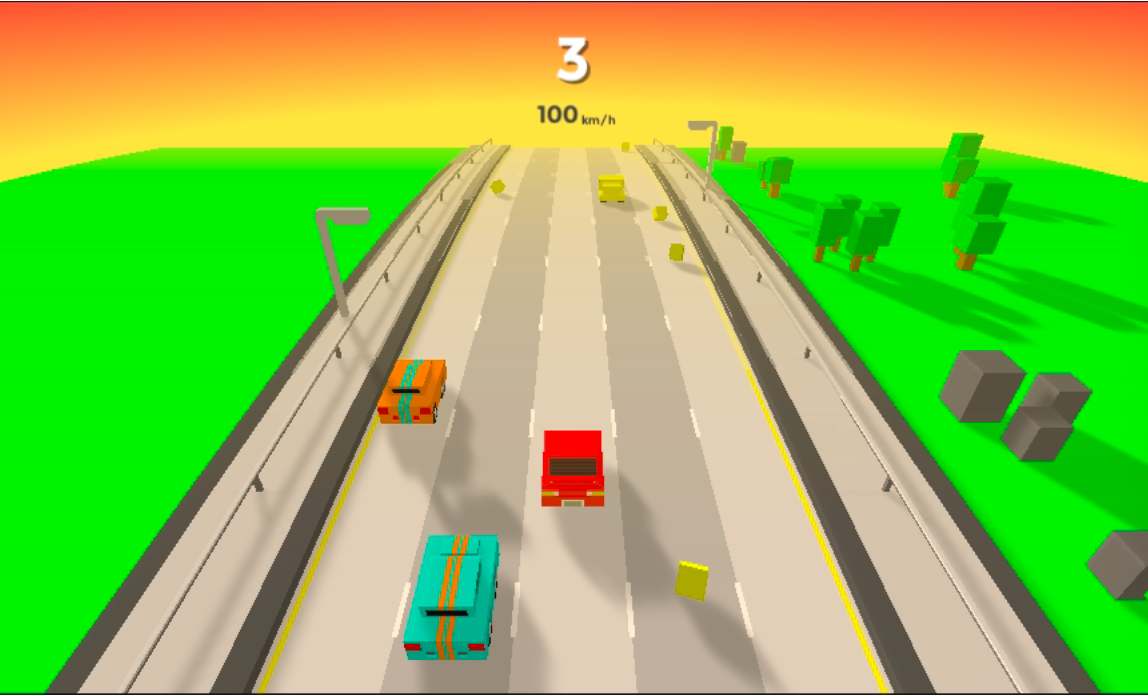}
\caption{\footnotesize Backseat Driver}
\label{fig: backseat driver ingame}
\end{wrapfigure}

\textbf{BDr Experiment 2: Human Termination.} To complement our results, we evaluated human termination on Backseat Driver. For this, we generated data of termination sequences from agents of varying quality (ranging from random to expert performance). We asked five human supervisors to label subsequences of this data by terminating the agent in situations of ``continual discomfort". This guideline was kept ambiguous to allow for diverse termination signals. The final dataset consisted of 512 termination examples. We then trained a model to predict human termination and implemented it into Backseat Driver to simulate termination. We refer the reader to \Cref{appendix: implementation details} for specific implementation details. \Cref{fig: backseat driver} shows results for human termination in Backseat Driver. As before, a significant performance increase was evident in our experiments. Additionally, we found that using a cost penalty (purple) or termination penalty (brown) for TermPG did not greatly affect performance.

\begin{figure}[t!]
\centering
\includegraphics[width=0.4\linewidth]{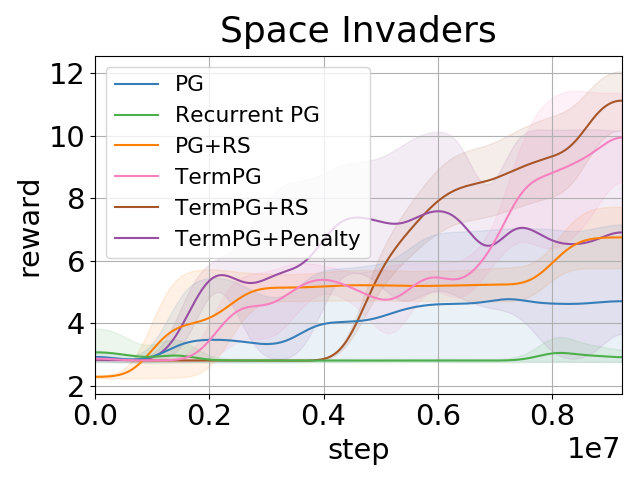}
\includegraphics[width=0.4\linewidth]{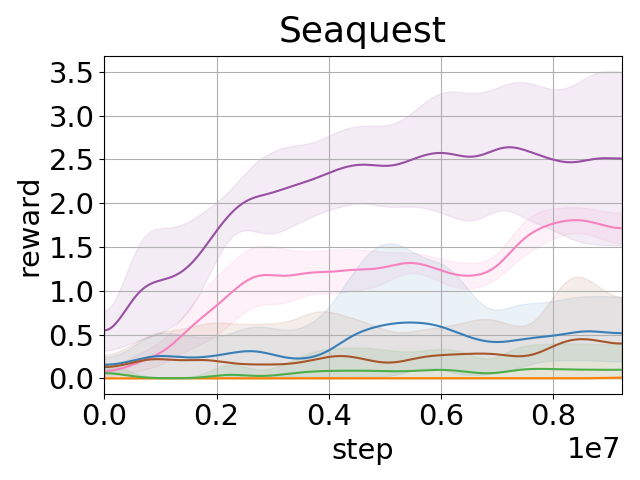}
\includegraphics[width=0.4\linewidth]{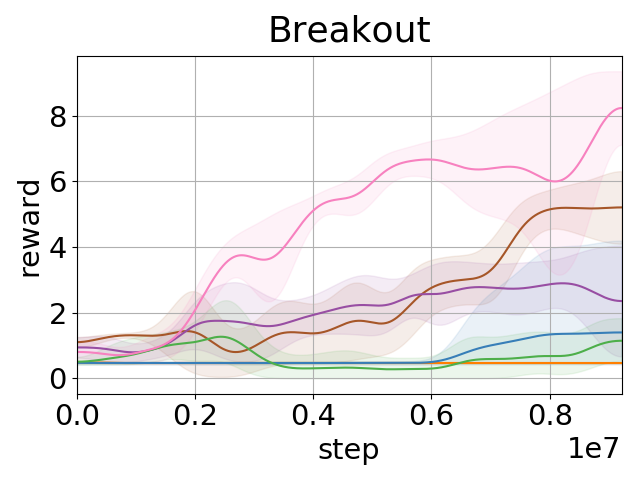}
\includegraphics[width=0.4\linewidth]{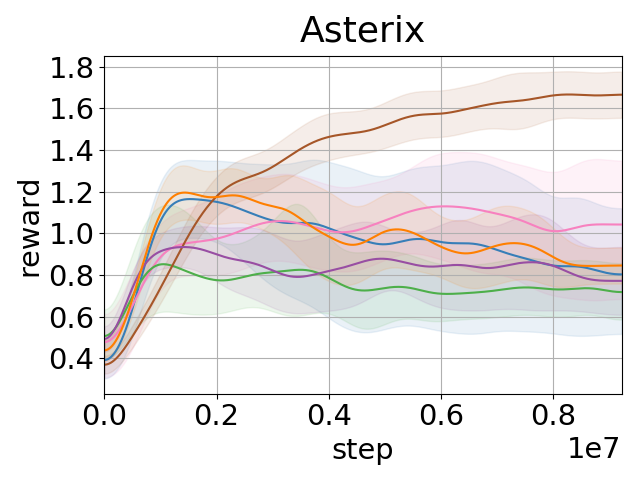}
\caption{\small Results for MinAtar benchmarks. All runs were averaged over five seeds. Comparison of best performing TermPG variant to best performing PG variant (relative improvement percentage): $80\%$ in Space Invaders, $150\%$ in Seaquest, $410\%$ in Breakout, and $90\%$ in Asterix.}
\label{fig: minatar}
\end{figure}

\begin{table*}[t!]
\caption{\label{table: results} \small Summary of results (top) and ablations for TermPG (bottom). Standard deviation optimism did not have significant impact on performance. Removing optimism or the dynamic discount factor had negative impact on performance. TermPG was found to be robust to model misspecifcations of the accumulated cost window.}
\centering
\hspace*{-0.8cm}  
\begin{scriptsize}
\begin{tabular}{|c|cc|cccc|}
\hline 
\multicolumn{1}{|c|}{} & \multicolumn{2}{c|}{\bf {\small Backseat Driver} } & \multicolumn{4}{c|}{\bf {\small MinAtar}} \\
\hline \hline 
\bf {\small Experiment} & \bf {\small Coin Avoid.}  & \bf{\small Human} &  \bf {\small Space Inv.}   & \bf {\small Seaquest}   & \bf {\small Breakout} & \bf {\small Asterix}    \\ \hline
PG & 
$5.3 \pm 0.8$   & $4.9 \pm 1.5$    & 
$5.2 \pm 1.8$   & $0.6 \pm 0.4$  & $1.4 \pm 2.8$ & $0.8 \pm 0.3$   \\ \hline
Recurrent PG & 
$3.4 \pm 0.21$   & $5 \pm 1.8$    & 
$2.8 \pm 0.05$   & $0.1 \pm 0.3$  & $0.7 \pm 0.6$ & $0.7 \pm 0.2$   \\ \hline
PG + RS & 
$5.9 \pm 1.4$   & $7.4 \pm 1.7$    & 
$7.6 \pm 2.3$   & $0.4 \pm 0.2$  & $0.5 \pm 0.03$ & $0.9 \pm 0.2$   \\ \hline
\rowcolor{Gray}
TermPG (ours) & 
$\mathbf{8.7 \pm 1.4}$ & $8.3 \pm 1.3$ &
$\mathbf{9.7 \pm 1.1}$   & $\mathbf{1.4 \pm 0.8}$  & $\mathbf{8.2 \pm 0.3}$ & $1 \pm 0.2$   \\ \hline
\rowcolor{Gray}
TermPG + RS (ours) & 
$\mathbf{8.4 \pm 1.3}$   & $7.7 \pm 0.3$    & 
$\mathbf{11.8 \pm 0.8}$   & $0.3 \pm 0.6$  & $5.1 \pm 1$ & $0.8 \pm 0.1$   \\ \hline
\rowcolor{Gray}
TermPG + Penalty (ours) & 
$6 \pm 0.8$ & $\mathbf{11.8 \pm 1.5}$ &
$7.7 \pm 1.4$   & $\mathbf{2.4 \pm 1}$  & $2.3 \pm 2.3$ & $\mathbf{1.7 \pm 0.1}$   \\ 

\hline\hline
\bf {\small Ablation Test} & \bf {\small Coin Avoid.}   & \bf {\small Human} & \bf {\small Space Inv.}   & \bf {\small Seaquest}   & \bf {\small Breakout} & \bf {\small Asterix}    \\ \hline
Optimism with Ensemble Std. & 
$7.6 \pm 2.1$ & $7.5 \pm 1.1$ &
$2.8 \pm 0.02$    & $0.9 \pm 0.6$   &  $10.9 \pm 1$   & $1 \pm 0.1$ \\ \hline
No Optimism & 
$7.8 \pm 1.3$   & $8.8 \pm 1.2$    & 
$5.2 \pm 1.6$   & $0.7 \pm 0.3$  & $1.3 \pm 0.7$ & $1 \pm 0.1$   \\ \hline
No Dynamic Discount & 
$6.9 \pm 0.6$   & $5.9 \pm 0.7$    & 
$4.4 \pm 1.8$   & $0.4 \pm 0.1$  & $0.5 \pm 0.02$ & $0.8 \pm 0.1$   \\ \hline
$\times 0.5$ Window Misspecification & 
$7.2 \pm 1.1$   & $7.2 \pm 0.5$    & 
$9.7 \pm 3.1$   & $2.4 \pm 0.4$  & $7.9 \pm 0.8$ & $0.8 \pm 0.1$   \\ \hline
$\times 2$ Window Misspecification & 
$8.3 \pm 0.1$   & $8.4 \pm 0.2$    & 
$11.1 \pm 3$   & $2.2 \pm 0.2$  & $10.3 \pm 0.8$ & $1 \pm 0.1$   \\ 
\hline
\end{tabular}
\end{scriptsize}
\end{table*}%

\textbf{MinAtar. } We further compared our method to the PG, recurrent PG, and reward shaping methods, on MinAtar \citep{young19minatar}. For each environment, we defined cost functions that do not necessarily align with the pre-specified reward, to mimic uncanny behavior that humans are expected to dislike. For example, in Breakout, the agent was penalized whenever the paddle remained in specific regions (e.g., sides of the screen), whereas in Space Invaders, the agent was penalized for ``near misses" of enemy bullets. We refer the reader to \Cref{appendix: implementation details} for specific details of the different termination cost functions.

\Cref{fig: minatar} depicts results on MinAtar. As with Backseat Driver, TermPG lead to significant improvement, often achieving a magnitude order as much reward as Recurrent PG. We found that adding a termination penalty and cost penalty produced mixed results, with them being sometimes useful (e.g., Space Invaders, Sequest, Asterix), yet other times harmful to performance (e.g., Breakout). Therefore, we propose to fine-tune these penalties in \Cref{alg: Termination PG}. Finally, we note that training TermPG was, on average, $67\%$ slower than PG, on the same machine. Nevertheless, though TermPG was somewhat more computationally expensive, it showed a significant increase in overall performance. A summary of all of our results is presented in \Cref{table: results} (top).

\textbf{Ablation Studies.} We present various ablations for TermPG in \Cref{table: results} (bottom). First, we tested the effect replacing the type of cost optimism in TermPG. In \Cref{sec: termpg}, cost optimism was defined using the minimum of the cost ensemble, i.e., $\text{min}\brk[c]*{c_{\phi_m}}$. Instead, we replaced the cost optimism to $C_{\text{optimistic}} = \text{mean}\brk[c]*{c_{\phi_m}} - \alpha \text{std}\brk[c]*{c_{\phi_m}}$, testing different values of $\alpha$. Surprisingly, this change mostly decreased performance, except for Breakout, where it performed significantly better. Other ablations included removing optimism altogether (i.e., only using the mean of the ensemble), and removing the dynamic discount factor. In both cases we found a significant decrease in performance, suggesting that both elements are essential for TermPG to work properly and utilize the estimator of the unknown costs. Finally, we tested misspecifications of our model by learning with windows that were different from the environment's real cost accumulation window. In both cases, TermPG was suprisingly robust to window misspecification, as performance remained almost unaffected by it.

\section{Related Work}

Our setup can be linked to various fields, as listed below.

\textbf{Constrained MDPs. } Perhaps the most straightforward motivations for external termination stems from constraint violation \citep{chow2018lyapunov,efroni2020exploration,hasanzadezonuzy2020learning}, where strict or soft constraints are introduced to the agent, who must learn to satisfy them. In these setups, which are often motivated by safety \citep{garcia2015comprehensive}, the constraints are usually known. In contrast, in this work, the costs are \emph{unknown} and only implicit termination is provided.

\textbf{Reward Design. } Engineering a good reward function is a hard task, for which frequent design choices may drastically affect performance \citep{oh2021creating}. Moreover, for tasks where humans are involved, it is rarely clear how to engineer a reward, as human preferences are not necessarily known, and humans are non-Markovian by nature \citep{clarke2013human,christiano2017deep}. Termination can thus be viewed as an efficient mechanism to elicit human input, allowing us to implicitly interpret human preferences and utility more robustly than trying to specify a reward.

\textbf{Global Feedback in RL. } Recent work considered once-per-trajectory reward feedback in RL, observing either the cumulative rewards at the end of an episode \citep{efroni2020reinforcement,cohen2021online} or a logistic function of trajectory-based features \cite{chatterji2021theory}. While these works are based on a similar solution mechanism, our work concentrates on a new framework, which accounts for non-Markovian termination. Additionally, we provide per-state concentration guarantees of the unknown cost function, compared to global concentration bounds in previous work \citep{abbasi2011improved,zhang2016online,qi2018bandit,abeille2021instance}. Using our local guarantees, we are able to construct a scalable policy gradient solution, with significant improvement over recurrent and reward shaping based approaches. 

\textbf{Preference-based RL. } In contrast to traditional reinforcement learning, preference-based reinforcement learning (PbRL) relies on subjective opinions rather than numerical rewards. In PbRL, preferences are captured through probabilistic rankings of trajectories \citep{wirth2016model,wirth2017survey,xu2020preference}. Similar to our work, \citet{christiano2017deep} use a regression model to learn a reward function that could account for the preference feedback. Our work considers a different setting in which human feedback is provided through termination, where termination and reward may not align.

\section{Discussion}

This paper formulated a new model to account for history-dependent exogenous termination in reinforcement learning. We defined the TerMDP framework and proposed a theoretically-guaranteed solution, as well as a practical policy-gradient approach. Our results showed significant improvement of our approach over various baselines. We stress that while it may seem as if the agent has two potentially conflicting goals---avoiding termination and maximizing reward---they are, in fact, aligned. The long-term consequences of actions need to account for longer survival which, in turn, allows for more reward collection. In what follows, we discuss $\kappa$, as factored in our regret bounds, as well possible limitations of our work.

\paragraph{The Role of $\kappa$}

As shown in \Cref{theorem: main}, $\kappa$ plays a significant role in the regret bound of \Cref{alg: Termination CRL}. This linear dependence is induced from the confidence bounds of \Cref{thm: local cost confidence front}. Informally, $\kappa$ is negligible whenever the costs $c$ and bias $b$ are ``well behaved". Suppose $\sum_{t=1}^h c_t(s_t^k,a_t^k) - b \gg 0$. In this case, $\kappa$ would be large and termination would mostly occur after the first step. As such, estimation of the costs would be hard (see \citet{chatterji2021theory}). Alternatively, suppose $\sum_{t=1}^h c_t(s_t^k,a_t^k) - b \ll 0$. In this case, $\kappa$ would also be large. Here, credit assignment would make the cost estimation problem harder, as trajectories would span longer horizons. It is unclear, as to the writing of this work, if other solutions to TerMDPs could bring about stronger regret guarantees that significantly reduce their dependence on $\kappa$. We note that lower bounds, which include $\kappa$, have previously been established for the estimation problem (see \citet{abeille2021instance,faury2020improved,jun2021improved}). Nevertheless, when searching for a policy which maximizes reward, it is unclear if estimation of the costs is indeed necessary for every state. We leave this direction for future work.

\paragraph{Limitations and Negative Societal Impact}

A primary limitation of our work involves the linear dependence of the logistic termination model. In some settings, it might be hard to capture true human preferences and behaviors using a linear model. Nevertheless, when measured across the full trajectory, our empirical findings show that this model is highly expressive, as we demonstrated on real human termination data (\Cref{sec: experiments}). Additionally, we note that work in inverse RL \citep{arora2021survey} also assumes such linear dependence of human decisions w.r.t. reward. Future work can consider more involved hypothesis classes, building upon our work to identify the optimal tradeoff between expressivity and convergence rate.

Finally, we note a possible negative societal impact of our work. Termination is strongly motivated by humans interacting with the agent. This may be harmful if not carefully controlled, as learning incorrect or biased preferences may, in turn, result in unfavorable consequences, or if humans engage in adversarial behavior in order to mislead an agent. Our work discusses initial research in this domain. We encourage caution in real-world applications, carefully considering the possible effects of model errors, particularly in applications that affect humans.

\newpage

\bibliography{bibliography}
\bibliographystyle{plainnat}

\clearpage

\include{checklist}

\clearpage

\appendixtrue

\addcontentsline{toc}{section}{Appendix}
\appendix
\part{}


{\hypersetup{linkcolor=black}
\parttoc}

\newpage


\section*{\uppercase{Organization of the Appendix}}
This appendix is organized as follows. We begin by further discussing motivations of our setting in autonomous driving and recommender system tasks in \Cref{appendix: motivation}. The first part of the appendix is then mostly focused on the empirical aspects, while the second part is mostly focused on the theoretical aspects, as well as missing proofs.

In \Cref{appendix: known costs} we discuss the TerMDP model in which costs are known. In this setting, we show that the costs are indeed sufficient for finding an optimal policy. That is, one need not care about the full history to account for termination, and only the current state and accumulated costs are needed to identify an optimal policy (i.e., achieve the same value as the optimal history-dependent policy).

In \Cref{appendix: termination bellman equations} we discuss the dynamic discount factor and derive the corresponding Termination Bellman Equations, as presented in \Cref{sec: discount factor}. In \Cref{appendix: implementation details} we discuss implementation details, including descriptions of the different environment and cost functions, TermPG implementations details and hyperparameters, and the construction of the human termination data. In \Cref{appendix: additional results} we provide some additional experimental results on the cost error and adding a cost bonus to the reward.

\Cref{appendix: planning} is focused on the problem of approximate planning in known TerMDPs. Here, we show that, by discretizing the costs on a grid, we can achieve near optimal performance through an approximate optimistic planner.

Finally, \Cref{supp: good event,appendix: regret analysis,supp: cost concentration,appendix: userful lemmas} provide full proofs of \Cref{thm: local cost confidence front,theorem: main}. In \Cref{supp: good event} we define failure events and bonuses for optimism in \Cref{alg: Termination CRL}. \Cref{appendix: regret analysis} analyzes and proves the regret guarantees in \Cref{theorem: main}, and \Cref{supp: cost concentration} provides proof for \Cref{thm: local cost confidence front}. Finally in \Cref{appendix: userful lemmas} we state auxiliary lemmas used throughout our analysis.

\newpage
\section{Motivation}
\label{appendix: motivation}

We begin by describing two concrete examples in which non-Markovian termination naturally occurs: overrides in autonomous driving and users abandoning recommender systems. 

\textbf{Autonomous Driving. } A myriad of factors affect the quality of a driving policy. These include safety, navigation efficiency, comfort, legislation, as well as specific preferences of the passengers. As these may be difficult to characterize quantitatively, heuristic metrics are derived and optimized. When a policy is released, these unknown factors may be enforced by a driving passenger, overriding the policy. For instance, a passenger might perceive a safe driving as dangerous due to differences between human and autonomous vehicle capabilities, thus halting the driving policy. This complication may be resolved by learning how to overcome such termination from previous occurrences.

\textbf{Recommender Systems. } Companies that recommend items to users may find contrasting goals between maximizing revenue and overall user satisfaction. For example, top-selling items which are recommended repeatedly can antagonize the user. Other popular items might offend the user or undermine the recommendation engine's credibility. Similar to the autonomous driving problem, user abandonment constitutes a critical signal, which may be caused by continual user dissatisfaction. The original criteria should therefore be optimized while also learning and accounting for these non-Markovian and unknown preferences.

\newpage
\section{TerMDPs with Known Costs}
\label{appendix: known costs}

In this section we define the TerMDP model with known costs. We show that the accumulated costs at time step $h$, i.e., $C_h = \sum_{t=1}^{h-1} c(s_t, a_t)$ and the current state $s_h$ are a sufficient statistic for computing an optimal policy $\pi^*$. That is, the history $\tau_{1:h} = (s_1, a_1, s_2, a_2, \hdots, s_{h-1}, a_{h-1})$ can be replaced by the accumulated costs $C_h$. To see this, we define an equivalent MDP, for which the state space is augmented by the accumulated costs, and show that an optimal policy of the augmented MDP achieves the optimal value of a history dependent optimal policy of the TerMDP.

We define the augmented MDP $\M_{\text{aug}} = \brk*{\sset_{\text{aug}}, \aset_{\text{aug}}, P_{\text{aug}}, R_{\text{aug}}, H}$, where $\sset_{\text{aug}} = \sset \times \R$ is the augmented state space, and $\aset_{\text{aug}} = \aset$ is the (unchanged) action space. The augmented transition function is defined for $s, C \in \sset \times \R, a \in \aset, s', C' \in \sset \times \R$
\begin{align*}
    P_{\text{aug}}(s', C' | s, C, a)
    =
    \indicator{C' = C + c(s,a)} 
    \cdot
    \begin{cases}
        1 & ,s'=s=\terminalstate \\
        \rho(C) & ,s \neq s' = \terminalstate \\
        (1-\rho(C)) P(s'|s,a)& ,s,s' \neq \terminalstate \\
        0 &, \mbox{o.w}.
    \end{cases}.
\end{align*}
Finally, the augmented reward function $R_{\text{aug}}$ satisfies $r_{\text{aug},h}(s_{\text{aug},h}^k=(s_h^k, C), \tilde{a}_h^k) = r_h(s_h^k,a_h^k)$.

Next, consider the TerMDP with known costs $\terMDP=(\sset,\aset, P,R,H,c)$, and let $C_h = \sum_{t=1}^{h-1} c(s_t, a_t)$. Then,
\begin{align*}
    P(s_{h+1}, C_{h+1} | s_h, a_h, \tau_{1:h})
    &=
    P(s_{h+1}, C_{h+1} | s_h, a_h, \tau_{1:h}, C_h) \\
    &=
    \indicator{C_{h+1} = C_h + c(s_h,a_h)}
    \cdot
    \begin{cases}
        1 & ,s_{h+1}=s_h=\terminalstate \\
        \rho(C_h) & ,s_h \neq s_{h+1} = \terminalstate \\
        (1-\rho(C_h)) P(s_{h+1}|s_h,a_h)& ,s_h,s_{h+1} \neq \terminalstate \\
        0 &,\mbox{o.w}.
    \end{cases} \\
    &=
    P_{\text{aug}}(s_{h+1}, C_{h+1} | s_h, C_h, a_h)
    =
    P_{\text{aug}}(s_{\text{aug},h+1} | s_{\text{aug},h}, a_h).
\end{align*}

To prove that the costs are sufficient for optimality, we prove a more general result. Particularly, define an MDP $(\sset_1 \times \sset_2, \A, P, r, H)$, and let $f: \sset_2 \mapsto D$, where $D$ is some known domain. Define the following set of deterministic policies
\begin{align*}
    \Pi_{\text{aug}} = \brk[c]*{\pi: \sset_1 \times \sset_2 \mapsto \A :  \exists \eta:\sset_1 \times D \mapsto [0,1], \pi(s_1, s_2) = \eta(s_1, f(s_2) }.
\end{align*}
Define the augmented optimal value for some $s \in \sset_1 \times \sset_2$
\begin{align*}
    V^*_{\text{aug},1}(s_1, s_2) 
    = 
    \max_{\pi \in \Pi_{\text{aug}}}
    \E{\sum_{t=1}^H r_t(s_t, a_t) | s_1=s_1, s_2 = s_2, a_t \sim \pi_t(s_1, s_2)}.
\end{align*}
We will show that if the reward does not depend on $\sset_2$, and if $P(s'_1, f(s'_2) | s_1, s_2, a) = P(s'_1, f(s'_2) | s_1, f(s_2), a)$, then $V^*_{\text{aug},1}(s_1, s_2) = V^*_{1}(s_1, s_2)$. This will prove that costs are indeed sufficient, as playing any policy in $\Pi_{\text{aug}}$ in the original MDP and achieve the same value. To see this, choose $\sset_2$ to be the set of possible trajectories in the known TerMDP, and let
\begin{align*}
    f(\tau_{1:h}) 
    = 
    f(s_1, a_1, s_2, a_2, \hdots s_{h-1}, a_{h-1})
    =
    \sum_{t=1}^{h-1} c(s_t, a_t).
\end{align*}
Then, since we previously showed that $r_{\text{aug},h}(s_{\text{aug},h}^k=(s_h^k, C), \tilde{a}_h^k) = r_h(s_h^k,a_h^k)$ and $P(s_{h+1}, C_{h+1} | s_h, a_h, \tau_{1:h}) = P_{\text{aug}}(s_{\text{aug},h+1} | s_{\text{aug},h}, a_h)$, this concludes our claim. The formal result is stated and proved below. 

\begin{proposition}
    Let $\M = (\sset_1 \times \sset_2, \A, P, r, H)$. Assume for any $s_1, s_2 \in \sset_1 \times \sset_2$, $a \in \A$, $P(s'_1, f(s'_2) | s_1, s_2, a) = P(s'_1, f(s'_2) | s_1, f(s_2), a)$ and $r(s_1, s_2, a) = g(s_1, a)$, for some deterministic function ${g:\sset_1 \times \A \mapsto [0,1]}$. Then, for any $s_1, s_2 \in \sset_1 \times \sset_2$,
    \begin{align*}
        V^*_{\text{aug},1}(s_1, s_2) = V^*_1(s_1, s_2).
    \end{align*}
\end{proposition}
\begin{proof}
    We prove by induction on $h \in [H]$. For $h = H$, the result follows trivially since $V^*_{\text{aug},H}(s_1, s_2) = V^*_H(s_1, s_2) = \max_a r(s_1, s_2, a)$. Next, assume that $V^*_{\text{aug},k+1}(s_1, s_2) = V^*_{k+1}(s_1, s_2)$ for some $k \in \brk[c]*{1, \hdots, H-1}$ and all $s_1, s_2 \in \sset_1 \times \sset_2$. Then, by the Bellman Equations,
    \begin{align*}
        V^*_k(s_1, s_2)
        &=
        \max_{a \in \A} 
        r(s_1, s_2, a) 
        + 
        \sum_{s'_1, s'_2 \in \sset_1 \times \sset_2} P(s'_1, s'_2 | s_1, s_2, a) 
        V^*_{k+1}(s'_1, s'_2) \\
        &\overset{\text{induction step}}{=}
        \max_{a \in \A} 
        r(s_1, s_2, a) 
        + 
        \sum_{s'_1, s'_2 \in \sset_1 \times \sset_2} P(s'_1, s'_2 | s_1, s_2, a) 
        V^*_{\text{aug},k+1}(s'_1, s'_2) \\
        &=
        \max_{a \in \A} 
        g(s_1, a) 
        +
        \sum_{s'_1 \in \sset_1}
        \sum_{C \in D}\sum_{s'_2 : f(s'_2) = C}
        P(s'_1, s'_2 | s_1, s_2, a) 
        V^*_{\text{aug},k+1}(s'_1, s'_2)
\end{align*}
Next, by \Cref{lemma: dependence on f}, there exists $U^*_{k+1}: \sset_1 \times D \mapsto \R$ such that $V^*_{\text{aug},k+1}(s_1, s_2) = U^*_{k+1}(s_1, f(s_2))$, for all $s_1, s_2 \in \sset_1 \times \sset_2$. Then,
\begin{align*}
    \sum_{C \in D}\sum_{s'_2 : f(s'_2) = C}
    P(s'_1, s'_2 | s_1, s_2, a) 
    V^*_{\text{aug},k+1}(s'_1, s'_2)
    &=
    \sum_{C \in D}\sum_{s'_2 : f(s'_2) = C}
    P(s'_1, s'_2 | s_1, s_2, a) 
    U^*_{k+1}(s'_1, f(s'_2)) \\
    &=
    \sum_{C \in D}
    U^*_{k+1}(s'_1, C)
    \sum_{s'_2 : f(s'_2) = C}
    P(s'_1, s'_2 | s_1, s_2, a) \\
    &=
    \sum_{C \in D}
    U^*_{k+1}(s'_1, C)
    P(s'_1, f(s'_2)=C | s_1, s_2, a).
\end{align*}
Using the assumption, we have that $P(s'_1, f(s'_2)=C | s_1, s_2, a) = P(s'_1, f(s'_2)=C | s_1, f(s_2), a)$. Then,
\begin{align*}
    \sum_{C \in D}\sum_{s'_2 : f(s'_2) = C}
    P(s'_1, s'_2 | s_1, s_2, a) 
    V^*_{\text{aug},k+1}(s'_1, s'_2)
    &=
    \sum_{C \in D}
    U^*_{k+1}(s'_1, C)
    P(s'_1, f(s'_2)=C | s_1, f(s_2), a) \\
    &=
    \sum_{s_2' \in \sset_2}
    P(s'_1, s'_2 | s_1, f(s_2), a) 
    V^*_{\text{aug},k+1}(s'_1, s'_2).
\end{align*}
where the last step follows a similar argument as above. Combining the above we get that
\begin{align*}
        V^*_k(s_1, s_2)
        &=
        V^*_{\text{aug},k+1}(s'_1, s'_2) \\
        &=
        \max_{a \in \A} 
        g(s_1, a) 
        + 
        \sum_{s'_1, s'_2 \in \sset_1 \times \sset_2} P(s'_1, s'_2 | s_1, f(s_2), a) 
        V^*_{\text{aug},k+1}(s'_1, s'_2) \\
        &=
        \max_{\eta: \sset_1 \times D \mapsto [0,1]}
        g(s_1, \eta(s_1, f(s_2))) 
        + 
        \sum_{s'_1, s'_2 \in \sset_1 \times \sset_2} P(s'_1, s'_2 | s_1, f(s_2), \eta(s_1, f(s_2))) 
        V^*_{\text{aug},k+1}(s'_1, s'_2) \\
        &=
        \max_{\eta: \sset_1 \times D \mapsto [0,1]}
        r(s_1, s_2, \eta(s_1, f(s_2))) 
        + 
        \sum_{s'_1, s'_2 \in \sset_1 \times \sset_2} P(s'_1, s'_2 | s_1, s_2, \eta(s_1, f(s_2))) 
        V^*_{\text{aug},k+1}(s'_1, s'_2) \\
        &=
        \max_{\pi \in \Pi_{\text{aug}}}
        r(s_1, s_2, \pi(s_1, s_2)) 
        + 
        \sum_{s'_1, s'_2 \in \sset_1 \times \sset_2} P(s'_1, s'_2 | s_1, s_2, \pi(s_1, s_2)) 
        V^*_{\text{aug},k+1}(s'_1, s'_2) \\
        &=
        V^*_{\text{aug},k}(s_1, s_2)
    \end{align*}
    This completes the proof.
\end{proof}

\begin{lemma}
    \label{lemma: dependence on f}
    Let $\M = (\sset_1 \times \sset_2, \A, P, r, H)$. Assume for any $s_1, s_2 \in \sset_1 \times \sset_2$, $a \in \A$, $P(s'_1, f(s'_2) | s_1, s_2, a) = P(s'_1, f(s'_2) | s_1, f(s_2), a)$ and $r(s_1, s_2, a) = g(s_1, a)$, for some deterministic function ${g:\sset_1 \times \A \mapsto [0,1]}$. Then, for any $k \in [H]$, there exists $U^*_k: \sset_1 \times D \mapsto \R$ such that $V^*_{\text{aug},k}(s_1, s_2) = U^*_k(s_1, f(s_2))$, for all $s_1, s_2 \in \sset_1 \times \sset_2$.
\end{lemma}
\begin{proof}
    We prove by induction on $k$. For $k = H$, the result follows trivially as $V^*_{\text{aug},H}(s_1, s_2) = \max_a g(s_1, a)$ for all $s_1, s_2 \in \sset_1 \times \sset_2$. Otherwise, we the $V^*_{\text{aug},k+1}(s_1, s_2) = U^*(s_1, f(s_2))$ is true for some $k \in \brk[c]*{1, \hdots, H-1}$. Then,
    \begin{align*}
        V^*_{\text{aug},k}(s_1, s_2)
        &=
        \max_{\pi \in \Pi_{\text{aug}}}
        r(s_1, s_2, \pi(s_1, f(s_2)) 
        + 
        \sum_{s'_1, s'_2 \in \sset_1 \times \sset_2} P(s'_1, s'_2 | s_1, s_2, \pi(s_1, f(s_2)) 
        V^*_{k+1}(s'_1, s'_2) \\
        &\overset{\text{induction step}}{=}
        \max_{\pi \in \Pi_{\text{aug}}}
        r(s_1, s_2, \pi(s_1, f(s_2)) 
        + 
        \sum_{s'_1, s'_2 \in \sset_1 \times \sset_2} P(s'_1, s'_2 | s_1, s_2, \pi(s_1, f(s_2)) 
        U^*_{k+1}(s'_1, f(s'_2)) \\
        &=
        \max_{\pi \in \Pi_{\text{aug}}}
        g(s_1, \pi(s_1, f(s_2)) 
        + 
        \sum_{s'_1\in \sset_1} 
        \sum_{C \in D}
        \sum_{s'_2 : f(s'_2) = C}
        P(s'_1, s'_2 | s_1, s_2, \pi(s_1, f(s_2)) 
        U^*_{k+1}(s'_1, f(s'_2)) \\
        &=
        \max_{\pi \in \Pi_{\text{aug}}}
        g(s_1, \pi(s_1, f(s_2)) 
        + 
        \sum_{s'_1\in \sset_1} 
        \sum_{C \in D}
        U^*_{k+1}(s'_1, C) 
        \sum_{s'_2 : f(s'_2) = C}
        P(s'_1, s'_2 | s_1, s_2, \pi(s_1, f(s_2)) \\
        &=
        \max_{\pi \in \Pi_{\text{aug}}}
        g(s_1, \pi(s_1, f(s_2)) 
        + 
        \sum_{s'_1\in \sset_1} 
        \sum_{f(s'_2) \in D}
        P(s'_1, f(s'_2) | s_1, s_2, \pi(s_1, f(s_2)) 
        U^*_{k+1}(s'_1, f(s'_2)) \\
        &=
        \max_{\pi \in \Pi_{\text{aug}}}
        g(s_1, \pi(s_1, f(s_2)) 
        + 
        \sum_{s'_1\in \sset_1} 
        \sum_{f(s'_2) \in D}
        P(s'_1, f(s'_2) | s_1, f(s_2), \pi(s_1, f(s_2)) 
        U^*_{k+1}(s'_1, f(s'_2))
    \end{align*}
    By defining 
    \begin{align*}
        U^*_k(s_1, f(s_2)) 
        = 
        \max_{\pi \in \Pi_{\text{aug}}}
        g(s_1, \pi(s_1, f(s_2)) 
        + 
        \sum_{s'_1\in \sset_1} 
        \sum_{f(s'_2) \in D}
        P(s'_1, f(s'_2) | s_1, f(s_2), \pi(s_1, f(s_2)) 
        U^*_{k+1}(s'_1, f(s'_2)), 
    \end{align*}
    we see that $V^*_{\text{aug},k}(s_1, s_2) = U^*_k(s_1, f(s_2))$, for any $s_1, s_2 \in \sset_1, \sset_2$, and the proof is complete.
\end{proof}

\newpage
\section{Dynamic Discount Factor}
\label{appendix: termination bellman equations}

In this section we derive the Termination Bellman Equations, as defined in \Cref{sec: discount factor}. The state action value function for a TerMDP with known costs (\Cref{appendix: known costs}) is defined for $s, C \in \sset \times \R$, $a \in \aset$, $h \in [H]$, and any policy $\pi$ by
\begin{align*}
    Q_h^\pi(s, C, a)
    =
    \expect*{\pi}{
        \sum_{t=h}^H 
        \brk*{\prod_{j=h+1}^t \gamma_j } 
        r_t(s_t, a_t) 
        | s_h = s, C_h = C, a_h = a },
\end{align*}
where we denote
\begin{align}
\label{eq: discount factor}
    \gamma_j 
    = 
    1-\rho\brk*{C_j - b}.
\end{align}
Then, for any $h \in [H-1]$ we have that
\begin{align*}
    &Q_h^\pi(s, C, a) \\
    &=
    \expect*{\pi}{
    \sum_{t=h}^H 
    \brk*{\prod_{j=h+1}^t \gamma_j } 
    r_t(s_t, a_t) 
    | s_h = s, C_h = C, a_h = a } \\
    &=
    r_h(s,a)
    +
    \gamma_{h+1}
    \expect*{\pi}{
    \sum_{t=h+1}^H 
    \brk*{\prod_{j=h+2}^t \gamma_j } 
    r_t(s_t, a_t) 
    | s_h = s, C_h = C, a_h = a } \\
    &=
    r_h(s,a) + \gamma_{h+1} 
    \expect*{s' \sim P(\cdot | s,a), a' \sim \pi(s')}{
    \expect*{\pi}{
    \sum_{t=h+1}^H 
    \brk*{\prod_{j=h+2}^t \gamma_j } 
    r_t(s_t, a_t) 
    | \substack{ s_h = s, C_h = C, a_h = a, \\ s_{h+1} = s', C_{h+1} = C + s(s', a'), a_{h+1} = a'} }} \\
    &=
    r_h(s,a)
    +
    \gamma_{h+1}
    \expect*{s' \sim P(\cdot | s,a), a' \sim \pi(s')}{
    \expect*{\pi}{
    \sum_{t=h+1}^H 
    \brk*{\prod_{j=h+2}^t \gamma_j } 
    r_t(s_t, a_t) 
    | s_{h+1} = s', C_{h+1} = C + s(s', a'), a_{h+1} = a'}} \\
    &=
    r_h(s,a)
    +
    \gamma_{h+1}
    \expect*{s' \sim P(\cdot | s,a), a' \sim \pi(s')}{
    Q_{h+1}^\pi(s', C+c(s',a'), a')} \\
    &=
    r_h(s,a)
    +
    \brk*{1-\rho(C-b)}
    \expect*{s' \sim P(\cdot | s,a), a' \sim \pi(s')}{
    Q_{h+1}^\pi(s', C+c(s',a'), a')}.
\end{align*}
We get the Termination Bellman Equations
\begin{align*}
    Q_h^\pi(s, C, a)
    =
    r_h(s,a)
    +
    \brk*{1-\rho(C-b)}
    \expect*{s' \sim P(\cdot | s,a), a' \sim \pi(s')}{
    Q_{h+1}^\pi(s', C+c(s',a'), a')}.
\end{align*}

Similarly, we can define an infinite horizon setting, for which
\begin{align*}
    Q^\pi(s, C, a)
    =
    \expect*{\pi}{
        \sum_{t=0}^\infty 
        \brk*{\prod_{j=2}^t \gamma_j } 
        r_t(s_t, a_t) 
        | s_1 = s, C_1 = C, a_1 = a },
\end{align*}
which yields the infinite horizon Termination Bellman Equations
\begin{align*}
    Q^\pi(s,C,a) 
    =
    r(s,a) 
    + 
    \brk*{1-\rho(C-b)}\expect*{s' \sim P(\cdot | s,a), a' \sim \pi(s')}{Q^\pi(s', C + c(s', a'), a')}
\end{align*}

Having defined the Termination Bellman Equations, we use them for value estimation. Specifically, we consider estimating the value by minimizing the TD-error, for some $s, C, a, s', a'$
\begin{align*}
    \abs{
    Q(s,C,a) 
    - 
    r(s,a) 
    -
    \brk*{1-\rho(C-b)}
    {Q(s', C + c(s', a'), a')}}.
\end{align*}
In our implementation, we used Generalized Advantage Estimation (GAE, \citet{schulman2015high}), which uses an exponential average over estimates of the TD-error. We used the discount factor in \Cref{eq: discount factor} for the TD estimates. That is, an exponential average over
\begin{align*}
    \hat{A}^{(n)} = 
    \sum_{t=0}^\infty 
    \brk*{\prod_{j=2}^{n-1} \gamma_j } 
    r_t(s_t, a_t) 
    + 
    \brk*{\prod_{j=2}^{n} \gamma_j } 
    \hat{V}(s_n)
    -
    \hat{V}(s_1).
\end{align*}
See \citet{schulman2015high} for specific implementation details of GAE.

\newpage
\section{Implementation Details}
\label{appendix: implementation details}

In this section we describe in detail the implementation details of the used environments, cost functions, algorithm, and human termination procedure.

\subsection{Environments}

We begin by describing the environments and cost functions used in our paper. In all environments we used a bias $b = 6$ for termination, which was unknown to the training agent. For all our environments (except human termination, see \Cref{appendix: human}), we used a window size of $w = 30$.

\paragraph{Backseat Driver.}
The environment was built using Unity's MLAgents \citep{juliani2018unity}. We incorporated a reward and cost function in the environment as follows. A reward of $+0.1$ was given to the agent for any car that was onscreen and behind it (i.e., the agent was awarded for overtaking other vehicles). The environment was set to accumulate a cost of $+1$ for every coin that was collected. No negative reward was given by the environment upon death or termination. 

In Backseat Driver, the agent can take one of three actions: change lane left, change lane right, or stay in place. If the agent hits another car the episode is terminated and the agent is reset. The state is represented by a stack of three binary maps consisting of the agent position, the other vehicle positions, and the coin positions (top view). Similar to previous work, we used the past four time steps (stacked) as the agent's state. That is, at time $t$, the agent receives 12 binary maps consisting of the agent position, other vehicle positions, and coin positions for steps $t, t-1, t-2, t-3$. For efficient learning, we repeated every action four times in the environment.

\paragraph{MinAtar.} We used four of the original MinAtar benchmarks \citep{young19minatar}, enforcing termination using a predefined cost-function. For each of the environments, we used a cost function which was dependent on the agent's position w.r.t. objects and areas of the state. Specifically, we defined costs as follows.
\begin{itemize}
    \item \textbf{Space Invaders:} The agent controls a spaceship and must shoot other alien spacecrafts, while they also shoot bullets at the agent.
    
    The agent was penalized ($c = 1$) for bullets passing at distance $d = 1$ from its current position. That is, near-misses of enemy bullets penalized the player. Ideally, the agent must learn to play the game while avoiding dangerous states in which one erroneous action can lose the game.
    
    \item \textbf{Seaquest:} In this environment, the agent controls a submarine. The agent must shoot enemies, avoid hitting objects (enemies, fish, or bullets), and carefully replenish its oxygen by moving to the surface. 
    
    The agent was penalized ($c = 1$) whenever it was positioned mid-depth. That is, denoting by $D$ the maximum reachable depth, the agent was penalized whenever it was positioned at depth $D / 2$. Ideally, the agent would remain close to the surface or deep in the water, avoiding unnecessary switches, to minimize termination.
    
    \item \textbf{Breakout:} In this game, the agent controls a paddle and can move it (left or right), bouncing a ball, which can hit rows of bricks. The agent is rewarded for breaking these bricks. If the paddle misses the ball, the agent loses and resets.
    
    The agent was penalized ($c = 1$) whenever it was positioned at the far-most left or right positions of the screen, i.e., close to the edges of the screen. Ideally, the agent would remain close to the center of the screen, where it can quickly reach the ball, and avoid getting ``stuck" at the sides of the screen.
    
    \item \textbf{Asterix:} The agent can move in any of the directions: up, down, left, or right, while avoiding spawned enemies, and picking up treasure. If the player hits an enemy, the agent loses and resets.
    
    The agent was penalized ($c = 1$) whenever it was positioned at distance $d = 1$ from an enemy. Similar to previous environments, this cost function was used to encourage the agent to stay away from enemies. Ideally, the agent should not be too close to enemies, for which one erroneous action could lose the game. 
\end{itemize}

\subsection{TermPG}

TermPG (\Cref{alg: Termination PG}) uses two key elements - a policy gradient algorithm (\texttt{ALG-PG}), and a training procedure for learning the costs and designing a dynamic cost-dependent discount factor, as described in \Cref{sec: termpg}. For the policy gradient algorithm, we used Proximal Policy Optimization (PPO, see \citet{schulman2017proximal}), implemented in RLlib \citep{liang2018rllib}. We mostly used the default hyperparameters -- full specification is given in \Cref{table: ppo hyperparameters}. We used the same hyperparemeters for the TermPG variants and the standard PG variants described in \Cref{sec: experiments}, with one important exception -- the PG variants used a constant discount factor $\gamma = 0.99$, whereas the TermPG variants did not use a discount factor \footnote{We also used a constant discount factor of $\gamma = 0.99$ for the ablation test of removing the dynamic discount factor in TermPG (\Cref{sec: experiments}, \Cref{table: results}).}, as the dynamic discount factor was used instead (\Cref{sec: discount factor}). The agent used a convolutional neural network (CNN), with hyperparameters given \Cref{table: model hyperparameters}.

\begin{table*}[t!]
\caption{\label{table: ppo hyperparameters} Hyperparameters used to train PPO agent}
\centering
\hspace*{-0.8cm}  
\begin{tabular}{|c|c|c|}
\hline 
\toprule
    {\bfseries Name} & {\bfseries Value}  & {\bfseries Comments}\\
    \midrule\midrule[.1em]
    Batch size & $32$ &  \\
    \hline
    Learning rate & $\num{5e-5}$ &  \\
    \hline
    Rollout size & $1024$ &  \\
    \hline
    Num epochs & $3$ & How many training epochs to do after each rollout \\
    \hline
    Entropy coef. & $0$ &  \\
    \hline
    kl coef & $0.2$ & Initial coefficient for KL divergence \\
    \hline
    kl target & $0.01$ & Target value for KL divergence \\
    \hline
    GAE $\lambda$ & $1$ & The GAE (lambda) parameter \\
    \hline
    Num workers & $8$ & \\
    \hline
    \bottomrule
\end{tabular}
\end{table*}

\begin{table*}[t!]
\caption{\label{table: model hyperparameters} CNN architecture hyperparameters for agent policy and cost networks }
\centering
\hspace*{-0.8cm}  
\begin{tabular}{|c|c|}
\hline 
\toprule
    Input size & $42 \times 42$  \\ \hline
    Num Filters & $16, 32, 256$  \\ \hline
    Stride & $4, 4, 11$  \\ \hline
    Padding & $2, 2, 1$  \\ \hline
    Activations & Relu \\ \hline
    Post-network FC hidden & $400, 256$  \\ \hline
\bottomrule
\end{tabular}
\end{table*}

For learning the costs, we used a model of three identical cost networks. At every iteration, rollouts that were collected in the environments were added to a finite buffer (FIFO) for training the cost networks. We labeled the rollouts according to whether they ended in termination. For training the cost networks we then sampled a (different) rollout from the buffer for each of the cost networks. We then split each of the rollouts to their respective windows ($t^* - 1$ negative labels and $1$ positive label for a trajectory ending in termination). Finally, we trained them the cost networks end-to-end using the cross entropy loss (see \Cref{fig: cost diagram}). We repeated this procedure for $30$ steps. We used the same network for training the cost model as we used for the PPO agent, with hyperparameters given in \Cref{table: model hyperparameters}. Hyperparameters for the TermPG algorithm are given in \Cref{table: termpg hyperparameters}.

\begin{table*}[t!]
\caption{\label{table: termpg hyperparameters} TermPG hyperparameters }
\centering
\hspace*{-1.8cm}  
\begin{tabular}{|c|c|c|}
\hline 
\toprule
    {\bfseries Name} & {\bfseries Value}  & {\bfseries Comments}\\
    \midrule\midrule[.1em]
    Learning rate & $\num{1e-3}$ &  \\
    \hline
    Bonus coefficient & $1$ & Coefficient used for cost bonus (optimism) \\
    \hline
    Num Ensemble & $3$ &  \\
    \hline
    Cost-net train steps & $30$ &  \\
    \hline
    Cost-net batch size & $t^*$ &  Each network receives all windows in trajectory ($t^*$ examples for every network) \\
    \hline
    Replay buffer size & $1000$ & Number of trajectories held in buffer \\
    \hline
    Window size & $30$ &  \\
    \hline
    \bottomrule
\end{tabular}
\end{table*}

\subsection{Human Termination Data}
\label{appendix: human}

In this subsection we describe in detail our process of creating and training the human termination data for Backseat Driver. To begin, we used the Backseat Environment without termination (i.e., we considered the standard infinite horizon MDP). We trained an agent using PPO on the infinite horizon environment (without termination), saving snapshots of the agent during training. Once training was over we used the snapshots to generate trajectories from the different agents, constructing a large dataset of over 5000 trajectories of different quality.

To obtain termination feedback, we randomly sampled trajectories from the generated dataset and played them back to a human supervisor, who was instructed to terminate the agent. The instructions for termination were to terminate the agent whenever it drives in a manner that is uncomfortable. In many scenarios the human supervisors terminated the agent when it switched lanes back and forth for no apparent reason (e.g., not for overtaking other vehicles). We used five different human supervisors for labeling random trajectories in the dataset, generating a total of 512 positive termination examples.

Once the data generation was complete, we trained a classifier for the termination problem with different window size. We found that a window size of $w=35$ best fit the termination accuracy (on a validation set). Finally, we used the trained model as an estimated termination signal for Backseat Driver. We emphasize that, while we used the same termination window ($w = 35$) for training the TermPG agent, we ran experiments with $\times x2, \times 0.5$ misspecification of window size, showing that these did not hurt performance (see \Cref{sec: experiments}, \Cref{table: results}).

\newpage
\section{Additional Results}
\label{appendix: additional results}

\begin{figure}[t!]
\centering
\includegraphics[width=0.5\linewidth]{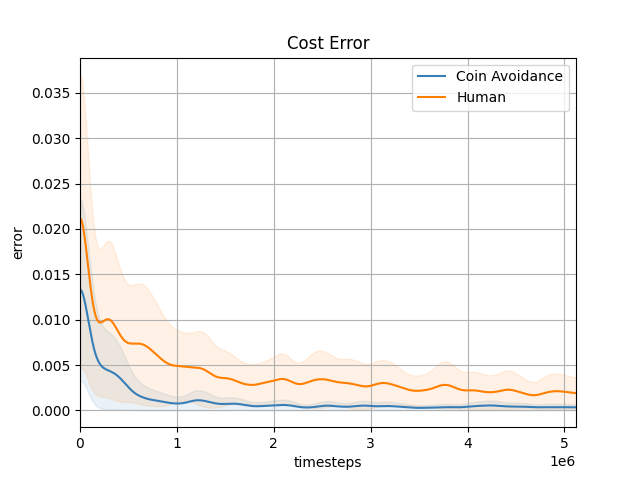}
\caption{\small Cost Error for Backseat Driver}
\label{fig: cost error}
\end{figure}

\paragraph{Cost Error.} We plot the cost error ($l_2$ norm) of the estimated costs in Backseat Driver in both the synthetic and human termination settings in \Cref{fig: cost error}. It can be seen that costs converge quickly, allowing for efficient learning. The speedy convergence suggests that TermPG can utilize the costs efficiently, enabling it to converge quickly to a good solution, as seen in our experiments in \Cref{sec: experiments}.

\paragraph{Cost Bonus.} In TermPG (\Cref{alg: Termination PG}) optimism is used by augmenting the state with an optimistic cost function. Nevertheless, this information is implicit, as it only passes through the state. We tested the affect of adding optimism in costs (to encourage exploration in areas of uncertainty in costs) directly through the reward function. Specifically, we used an uncertainty (defined by the ensemble of cost networks) to add a bonus to the reward
\begin{align*}
    r(s_t,a_t) \gets r(s_t,a_t) + \alpha U(C_t),
\end{align*}
where $U(C_t)$ is either the max-min difference or the standard deviation in the ensemble outputs, and $\alpha > 0$ is a hyperparameter for choosing the degree of optimism in the reward.

Results for training TermPG with a cost bonus are presented in \Cref{table: cost bonus}. Evidently, adding a cost bonus mostly hurt performance, even with small values of $\alpha$. These results suggest that adding additional cost bonus to the reward is not beneficial, and using optimism in the states is sufficient.

\begin{table*}[h!]
\caption{\label{table: cost bonus} Cost bonus for TermPG (other TermPG results are provided for reference.) }
\centering
\hspace*{-0.8cm}  
\begin{tabular}{|c|cc|cccc|}
\hline 
\multicolumn{1}{|c|}{} & \multicolumn{2}{c|}{\bf {\small Backseat Driver} } & \multicolumn{4}{c|}{\bf {\small MinAtar}} \\
\hline \hline 
\bf {\small Experiment} & \bf {\small Coin Avoid.}  & \bf{\small Human} &  \bf {\small Space Inv.}   & \bf {\small Seaquest}   & \bf {\small Breakout} & \bf {\small Asterix}    \\ \hline
\rowcolor{Gray}
TermPG & 
$\mathbf{8.7 \pm 1.4}$ & $8.3 \pm 1.3$ &
$\mathbf{9.7 \pm 1.1}$   & $\mathbf{1.4 \pm 0.8}$  & $\mathbf{8.2 \pm 0.3}$ & $1 \pm 0.2$   \\ \hline
\rowcolor{Gray}
TermPG + RS & 
$\mathbf{8.4 \pm 1.3}$   & $7.7 \pm 0.3$    & 
$\mathbf{11.8 \pm 0.8}$   & $0.3 \pm 0.6$  & $5.1 \pm 1$ & $0.8 \pm 0.1$   \\ \hline
\rowcolor{Gray}
TermPG + Penalty & 
$6 \pm 0.8$ & $\mathbf{11.8 \pm 1.5}$ &
$7.7 \pm 1.4$   & $\mathbf{2.4 \pm 1}$  & $2.3 \pm 2.3$ & $\mathbf{1.7 \pm 0.1}$ \\ \hline
Cost Bonus with $\alpha = 0.1$ & 
$6.6 \pm 0.45$ & $7.9 \pm 0.7$ & 
$6.5 \pm 1.46$ & $1.4 \pm 1.1$ & $2.2 \pm 0.2$ & $0.9 \pm 0.2$    \\ \hline   
Cost Bonus with $\alpha = 1$ & 
$7.3 \pm 0.7$ & $4.8 \pm 1.14$ & 
$2.8 \pm 0.02$ & $1.9 \pm 0.1$ & $0.2 \pm 0.05$ & $0.6 \pm 0.3$     \\ \hline  
\end{tabular}
\end{table*}%

\newpage
\section{Additional Theoretical Notations}

This section adds further notations needed for the sections that follow. Particularly, we define notations that are not provided in the paper for clarity, yet are beneficial for our theoretical analysis.

Throughout the paper, we work w.r.t. the natural filtration 
\begin{align*}
    \F_k
    =
    \sigma\brk*{\brk[c]*{(s_h^1,a_h^1,R_h^1)}_{h=1}^H,\dots,\brk[c]*{(s_h^k,a_h^k,R_h^k)}_{h=1}^H,s_1^{k+1}}.
\end{align*}

We use the notation $\tau_{1:h}^k$ when referring to the trajectory of the agent at the $k^{\text{th}}$ episode. For brevity, we denote $\tau^k=\tau_{1:H}^k$. Similarly, we denote the empirical visitation up to the $h^{\text{th}}$ time step $\hat d_h^k(t,s,a) = \indicator{t\leq h,s_t^k=s,a_t^k=a}$. With abuse of notation, we assume the bias term $b$ is consumed by the cost $c$, such that the first state always contains the bias term, i.e., $c(s_1, a_1) := c(s_1, a_1) - b$. Using this notation, we can write the termination probability after the $h^{\text{th}}$ time step in vector notation as $\rho\brk1{\sinner{\hat d_h^k,c}}$. We then define the total transition function by $T_h^{P, c}(s' | s_h^k, a_h^k, \tau_{1:h}^k) = \brk*{1-\rho\brk1{\sinner{\hat d_h^k, c}}}P_h(s' | s_h^k, a_h^k)$ for any non-terminal $s'\in\sset$ and $T_h^{P, c}(\terminalstate | s_h^k, a_h^k, \tau_{1:h}^k) = \rho\brk1{\sinner{\hat d_h^k, c}}$. For brevity, when working with the real kernel and costs $P,c$, we omit them from the notation and use $T_h$.  Finally we denote by $\bar X$ optimism for some quantity $X$, e.g., $\bar V$ is the optimistic value.

\newpage
\section{Approximate Planning in TerMDPs}
\label{appendix: planning}

We consider the following planning procedure for a known TerMDP $\terMDP=(\sset,\aset, P,R,H,c)$ (see \Cref{appendix: known costs}). 
For the approximate planning, we quantize the costs in this TerMDP to a resolution $\Delta c$.
Concretely, denote the lattice of resolution $\Delta x$ by $\mathcal{G}_{\Delta x} = \brk[c]*{\hdots, -2 \Delta x, \Delta x, 0, \Delta x, 2 \Delta x, \hdots}$. We then define the discretized cost function, for a discretization parameter $\Delta c$ as
\begin{align*}
    c_q = \floor*{\frac{c}{\Delta c}} \Delta c
\end{align*}
In this case, the accumulated cost must also lie on a grid such that $C = \sum_t c_{q,t} \in \mathcal{G}_{\Delta c}$. We denote by $c_{\max} = \max_{s,a} \abs{c(s,a)}$. Notice that $c_{\max} \leq L$, where $L$ is known, but $c_{\max}$ is usually much smaller. We denote the number of bins of the costs by $N$, which is generally bounded by $N \leq 2\frac{c_{\max}}{\Delta c} + 1$. For non-negative costs, we will show that a much smaller $N$ suffices (\Cref{appendix: non negative cost planning}). Then, due to the grid structure, the accumulated costs can obtain at most $HN$ different values. This allows us to perform standard planning procedure in episodic MDPs (using dynamic programming/value iteration) with a state space of size $NHS$, which requires computation complexity of $\Ob\brk*{H^3AS^2N^2}$ \citep{puterman2014markov}. In the rest of this section, we show how to choose $\Delta c$ so that this approximate planner will yield a near-optimal policy. Particularly, we show that $\Delta c = \Ob\brk{1/\sqrt{K}}$ maintains the same regret guarantees (namely, incurs lower-order regret penalty) while enabling tractable planning.

\subsection{Approximation Bounds for Quantized TerMDPs}

Denote the quantized TerMDP by $\terMDP^q=(\sset,\aset, P,R,H,c_q)$. For a policy $\pi$, we denote by $V^\pi, V_q^\pi$ its value in $\terMDP$ and $\terMDP^q$, respectively. 
We define a policy $\pi_q$ as the solution to the planning problem in the quantized TerMDP. In what follows we show that for any $\epsilon > 0$ one can ensure $V_1^*(s_1) - V_{1}^{\pi_q}(s_1) < \epsilon$ for some choice of $\Delta c$.

Notice that we always rounded-down the costs (decreased the termination probability) and, thus, $V^\pi\leq V_q^\pi$ for any $\pi$. Particularly, for $\pi^*$, we get that 
\begin{align*}
    V^*
    =
    V^{\pi^*}
    \overset{\text{optimism}}{\leq} 
    V_q^{\pi^*}
    \overset{\substack{\text{optimality of} \\ \text{ $\pi_q$ in $\terMDP^q$}}}{\leq} 
    V_q^{\pi_q} 
    = 
    V^{\pi_q} + \brk*{V_q^{\pi_q}-V^{\pi_q}}.
\end{align*}
Therefore, it is enough to show that $V_q^{\pi_q}-V^{\pi_q} < \epsilon$. We will show that a stronger results, that $V_q^{\pi}-V^{\pi} < \epsilon$ for any $\pi$.

By the value difference lemma (\Cref{lemma: value difference lemma}), for any policy $\pi$
\begin{align*}
    V_{q,1}^\pi(s_1) - V_1^\pi(s_1) 
    &= 
    \sum_{h=1}^H \E{\brk*{T_h^{P, c_q} - T_h^{P, c}}(\cdot\vert s_h,a_h,\tau_{1:h})^T
    V_1^\pi(\cdot, \tau_{1:h+1})} \\
    &\leq
    \sum_{h=1}^H \E{\norm{\brk*{T_h^{P, c_q} - T_h^{P, c}}(\cdot\vert s_h,a_h,\tau_{1:h})}_1\norm{V_1^\pi(\cdot, \tau_{1:h+1})}_{\infty}} \\
    &\leq
    H \sum_{h=1}^H \E{\norm{\brk*{T_h^{P, c_q} - T_h^{P, c}}(\cdot\vert s_h,a_h,\tau_{1:h})}_1} \\
    & =
    H\sum_{h=1}^H \E{\abs{\rho\brk1{\sinner{\hat d_h,c_q}} - \rho\brk1{\sinner{\hat d_h, c}}} \underbrace{\norm{P_h(\cdot\vert s_h,a_h)}_1}_{=1}+\abs{\brk*{1-\rho\brk1{\sinner{\hat d_h,c_q}}} - \brk*{1-\rho\brk1{\sinner{\hat d_h, c}}}}}\\
    & = 
    2H\sum_{h=1}^H \E{\abs{\rho\brk1{\sinner{\hat d_h,c_q}} - \rho\brk1{\sinner{\hat d_h, c}}}}
\end{align*}

\subsubsection{General Case: Signed Costs}
By the Lipschitz-continuity of the logistic function, we have that 
\begin{align*}
    V_{q,1}^\pi(s_1) - V_1^\pi(s_1) 
    &\leq
    2H\sum_{h=1}^H \E{\abs{\rho\brk1{\sinner{\hat d_h,c_q}} - \rho\brk1{\sinner{\hat d_h, c}}}} \\
    & \leq 
    \frac{H}{2} \sum_{h=1}^H \E{\abs{\sinner{\hat d_h,c_q-c}}} \\
    & \leq
    \frac{H}{2} \sum_{h=1}^H \E{\underbrace{\norm{\hat d_h}_1}_{=h}\underbrace{\norm{c_q-c}_\infty}_{\leq\Delta c}} \\ 
    & \leq
    \frac{H^3\Delta c}{2}
\end{align*}

We get that, for any $\epsilon > 0$, using $\Delta c \leq \frac{2\epsilon}{H^3}$ we get that
$
    V_{q,1}^\pi(s_1) - V_1^\pi(s_1) \leq \epsilon.
$ Note that, this choice of grid resolution induces $N = \Ob\brk*{H^3\frac{c_{\max}}{\epsilon}}$

\subsubsection{Better Approximation for Non-Negative Costs}
\label{appendix: non negative cost planning}

We remind the reader that in \Cref{appendix: additional results} we incorporated the bias $b$ into the costs $c$, such that $c(s_1, a_1) := c(s_1, a_1) - b$. This was helpful for utilizing inner product notations in our analysis. In this part, we assume the costs are positive, yet we do not limit the choice of bias $b$. Specifically, we assume that for all $h \geq 2$, and for all $s, a \in \sset \times \aset$, $c(s_h, a_h) \geq 0$. We further assume that for all $s, a \in \sset \times \aset$, $c(s_1, a_1) \geq - b$. This relates to a TerMDP with positive costs and an arbitrary bias $b$.

Let $C^* \in (0, Hc_{\max}]$ (will be explicitly chosen later). We show that, in the case of non-negative costs, it can be beneficial to clip the accumulated cost, once it is larger than $C^*$ (even if it is smaller than $Hc_{\max}$). Particularly, if $\sum_{t=1}^h c_{q,t}(s_t,a_t)\ge C^*$, we represent the accumulated cost by $C^*$. This, in turn, implies that the accumulated costs can have at most $N=\floor*{\frac{C^*+b}{\Delta c}}+1$ bins.

Notice that by the definition of the logistic function, for any $C_0$ and any $C\ge C_0$, it holds that
\begin{align*}
    \rho(C) \ge \rho(C_0) = \brk*{1+\exp(-C_0)}^{-1} \ge 1-\exp(-C_0).
\end{align*}
Therefore, since $\rho(C)\leq1$, we have for any $C\geq C_0$ that $\abs{\rho(C)-\rho(C_0}\leq \exp(-C_0)$.

Also, denote $h^* = \min\brk[c]*{h \in [H]:\sinner{\hat d_h,c_q} >C^*}$. Then,
\begin{align*}
    V_{q,1}^\pi(s_1) - V_1^\pi(s_1) 
    &\leq
    2H\sum_{h=1}^H \E{\abs{\rho\brk1{\sinner{\hat d_h,c_q}} - \rho\brk1{\sinner{\hat d_h, c}}}} \\
    &=
    2H \E{
    \sum_{h=1}^{h^*-1} \abs{\rho\brk1{\sinner{\hat d_h,c_q}} - \rho\brk1{\sinner{\hat d_h, c}}} 
    + 
    \sum_{h=h^*}^{H} \abs{\rho\brk1{\sinner{\hat d_h,c_q}} - \rho\brk1{\sinner{\hat d_h, c}}}} \\
    &\leq
    2H \E{
    \frac{1}{4} \brk*{h^*-1}^2 \Delta c
    + 
    (H-h^*) \exp\brk*{-C^*}
    } \\
    &\leq
    \frac{H^3 \Delta c}{2} + 2H^2\exp\brk*{-C^*}.
\end{align*}
In the second inequality, we also used the lipschitz property of the logistic function. Moreover, we used the fact that the costs are non-negative -- if the accumulated cost is larger than $C^*$ at $h^*$, then it is larger than $C^*$ for any $h\ge h^*$.

Finally, for any $\epsilon>0$, we set $\Delta c = \frac{\epsilon}{H^3}$ and $C^*=\log\brk*{\frac{4H^2}{\epsilon}}$, which ensure an error smaller than $\epsilon$. This parameter choice induces $N=\Ob\brk*{H^3\frac{\log\brk*{\frac{H^2}{\epsilon}}+b}{\epsilon}}$ -- potentially much smaller than the general case for which  we get $\Ob\brk*{H^3\frac{c_{\max}}{\epsilon}}$.

\subsection{Regret Bound for Approximate Planner}

We integrate the approximate planner into \Cref{alg: Termination CRL} and show that the regret bound is only mildly affected for $\epsilon=\Ob\brk*{1/\sqrt{K}}$ (or $\Delta c = \Ob\brk*{\frac{1}{H^3\sqrt{K}}}$). Particularly, instead of accurately solving the TerMDP $\brk1{\sset,\aset,H,\bar{r}^{k},\hat{P}^{k},\bar{c}^{k}}$, we apply our approximate planner and denote its output policy by $\pi_q^k$. Its value in the optimistic TerMDP is denoted by $\bar{V}_1^{\pi^k_q}(s_1^k)$.

Now, following the regret analysis of \Cref{appendix: regret analysis}, we have that

\begin{align*}
    \Reg{K}
    &= 
    \sum_{k=1}^K V_1^*(s_1^k) - V_1^{\pi^k_q}(s_1^k) \\
    &\leq
    \sum_{k=1}^K \bar{V}_1^{k}(s_1^k) - V_1^{\pi^k_q}(s_1^k) \\
    & =
    \sum_{k=1}^K \bar{V}_1^{\pi^k_q}(s_1^k) - V_1^{\pi^k_q}(s_1^k) + \sum_{k=1}^K \bar{V}_1^{k}(s_1^k) - \bar{V}_1^{\pi^k_q}(s_1^k)
\end{align*}
The first term can be bounded exactly as in the regret analysis of \Cref{appendix: regret analysis}, while the second bound is controlled by the approximation error of the planner, namely $\epsilon K$. Thus, choosing $\epsilon=\Ob\brk*{1/\sqrt{K}}$ leads to a negligible second term, allowing efficient planning while maintaining the stated regret bound of \Cref{theorem: main}.

\newpage
\section{Failure Events and Optimism}
\label{supp: good event}

In this section we focus on defining failure events and derive the bonuses described in \Cref{sec: theory}. Particularly, we will use these to show that the optimistic TerMDP (line 7 of \Cref{alg: Termination CRL}), used for planning, satisfies the needed optimism (i.e., larger than the optimal value), which will be used in the proof of \Cref{theorem: main}.

We define the following failure events.
{\small
\begin{align*}
    &F^r_k=\brk[c]*{\exists s,a,h:\ |r_h(s,a) - \hat{r}_h^k(s,a)| > \sqrt{\frac{2\log \frac{2SAHK}{\delta'}}{n_h^k(s,a)\vee 1}} }\\
    &F^p_k=\brk[c]*{\exists s,a,h:\ \norm{P_h(\cdot\mid s,a)- \hat{P}_h^k(\cdot\mid s,a)}_1 > \sqrt{\frac{4S\log\frac{3SAHK}{\delta'}}{n_h^k(s,a)\vee 1
    }}}\\
    &F^n = \brk[c]*{\sum_{k=1}^K \E{\sum_{h=1}^H \frac{1}{\sqrt{n_h^k(s^k_h,a^k_h)\vee 1}} \mid \F_{k-1}} > 16H^2\log\brk*{\frac{1}{\delta'}}+4SAH^2 +2\sqrt{2}\sqrt{SAH^2 K\log HK}}\\
    & F^c_k =  \brk[c]*{\exists s,a,h:\ \abs{c_h(s,a) - \hat c_h^k(s,a)} > 24\sqrt{\kappa SAH^{2.5}}(L+1)^{1.5}\log^2\brk*{\frac{4}{\delta'}\brk*{1+\frac{k(L+0.5)}{16S^2A^2\sqrt{H}}}}\frac{1}{\sqrt{n^k_h(s,a) + 4\frac{SAH}{L\sqrt{H}+0.5}}}}
\end{align*}}
Furthermore, the following relations hold.

\begin{itemize}
    \item Let $F^r=\bigcup_{k=1}^K F^r_k$. Then $\Pr\brk[c]*{F^r}\leq \delta'$, by Hoeffding's inequality, and using a union bound argument on all $s,a$, and all possible values of $n_{k}(s,a)$ and $k$. Furthermore, for $n(s,a)=0$ the bound holds trivially since $r\in[0,1]$. 
    \item Let $F^P=\bigcup_{k=1}^K F^{p}_k.$ Then $\Pr\brk[c]*{ F^p}\leq \delta'$, holds by \citep{weissman2003inequalities} while applying union bound on all $s,a$, and all possible values of $n_h^k(s,a)$ and $k$. Furthermore, for $n(s,a)=0$ the bound holds trivially. 
    \item $\Pr\brk[c]*{F^n}\leq \delta'$  by \Cref{lemma: expected cumulative visitation bound}.
    \item Let $F^{c} =\bigcup_{k=1}^K F^{c}_k$. Then $\Pr\brk[c]*{F^{c}}\leq \delta'$ for all $k \geq 1$ by \Cref{prop: local cost confidence}. 
\end{itemize}

We define the good event as the event where all failure events do not occur for all $k\in\brk[s]{K}$, namely, ${\G = \neg F^r \bigcap \neg F^p\bigcap \neg F^n \bigcap \neg F^c}$. Then, the following holds:

\begin{lemma}[Good event]
\label{lemma: good event}
Setting $\delta'=\frac{\delta}{4}$ then $\Pr\brk[c]{F^r \bigcup F^p\bigcup F^n \bigcup F^c}\leq \delta$. When the failure events does not hold we say the algorithm is outside the failure event, or inside the good event $\G$.
\end{lemma}

\subsection{Optimism}
\label{appendix: optimism}

Following the events in $\G$ (\Cref{lemma: good event}) we define the following bonuses.
\begin{align*}
    & b_k^r(h,s,a) = \sqrt{\frac{2\log \frac{8SAHK}{\delta}}{n_h^k(s,a)\vee 1}}\\
    & b_k^p(h,s,a) = H\sqrt{\frac{4S\log\frac{12SAHK}{\delta}}{n_h^k(s,a)\vee 1
    }}\\
    & b_k^c(h,s,a) = 24\sqrt{\kappa SAH^{2.5}}(L+1)^{1.5}\log^2\brk*{\frac{16}{\delta}\brk*{1+\frac{k(L+0.5)}{16S^2A^2\sqrt{H}}}}\frac{1}{\sqrt{n^k_h(s,a) + 4\frac{SAH}{L\sqrt{H}+0.5}}}\enspace.
\end{align*}
\newpage
The total reward bonus is defined as $b_k^{rp}(h,s,a) = b_k^r(h,s,a) + b_k^p(h,s,a)$. Adding the reward bonus and subtracting the cost bonus leads to an optimistic MDP, as we prove in the following lemma.

\begin{lemma}[Optimism]\label{lemma: optimism}
    Let $\bar{V}^k$ be the optimal value of the optimistic MDP $\OpterMDP \brk1{\sset,\aset,H,\bar{r}^{k},\hat{P}^{k},\bar{c}^{k}}$, clipped by $H$, i.e., for all $h \in [H], s \in \s$
    \begin{align*}
        \bar{V}^k_h(s, \tau_{1:h})
        =
        \min \brk[c]*{V^*_{\OpterMDP}(s, \tau_{1:h}), H}.
    \end{align*}
    Then, under the good event $\G$, for any $k\in\brk[s]{K}$, $s\in\brk[s]{S}$, history $\tau_{1:h}$ and $h\in\brk[s]{H}$, it holds that $\bar{V}_h^k(s,\tau_{1:h}) \ge V^*(s,\tau_{1:h})$
\end{lemma}
\begin{proof}
We prove the claim by an induction over $h$. 

First, the induction trivially holds for $h=H+1$, since $\bar{V}_h^k(s,\tau_{1:h}) = V^*(s,\tau_{1:h})=0$.

Now, let $h\in\brk[s]{H}$ and assume that the claim holds for $h'=h+1$, $\forall s\in \sset$ and $\tau_{1:h'}$.

Denote:
\begin{align*}
    & a^* \in \arg\max_a \brk[c]*{r_h(s,a) + T_h (\cdot \mid s,a, \tau_{1:h})^T V_{h+1}^*\brk*{\cdot, \tau_{1:h}\cup(s,a)}}.
\end{align*}
Recall that $\tau_{1:h}=(s_{h-1},a_{h-1},\dots,s_1,a_1)$ is the history up to time $h$, and for brevity, denote $\tau_{1:h+1}^*\triangleq(s,a^*)\cup\tau_{1:h}=(s,a^*, s_{h-1},a_{h-1}, \dots, s_1,a_1)$, the history up to time $h+1$ when visiting $s$ and playing $a^*$ on the $h^{th}$ timestep. Also, with some abuse of notation, we let $\brk[s]*{TV}(s,a,\tau) = T(\cdot\vert s,a,\tau)^T V(\cdot,(s,a)\cup \tau)$ and similarly define $\brk[s]*{PV}(s,a,\tau) = P(\cdot\vert s,a)^T V(\cdot,(s,a)\cup \tau)$. Finally, denote by $d^\tau$ the vector with elements $d^{\tau}(h,s,a) = 1$ if $(s_h,a_h)\in \tau$ and $0$ otherwise.
Then,
\begin{align*}
    \bar V_h^k&\brk*{s, \tau_{1:h}} - V_h^*\brk*{s, \tau_{1:h}} \\
    &= 
    \max_a  \brk[c]*{\bar r_h^k(s, a) + \brk[s]*{T_h^{\hat P_k, \bar c_k}\bar{V}_{h+1}^k}(s,a, \tau_{1:h})} 
    - \max_a \brk[c]*{r_h(s,a) + \brk[s]*{T_hV_{h+1}^*}(s,a, \tau_{1:h})}\\
    & \overset{(1)}\geq \bar r_h^k(s, a^*) + \brk[s]*{T_h^{\hat P_k, \bar c_k}\bar{V}_{h+1}^k}(s,a^*, \tau_{1:h}) 
    - r_h(s,a^*) -  \brk[s]*{T_hV_{h+1}^*}(s,a^*, \tau_{1:h}) \\
    & \overset{(2)}\geq \bar r_h^k(s, a^*) + \brk[s]*{T_h^{\hat P_k, \bar c_k}V_{h+1}^*}(s,a^*, \tau_{1:h})  
    - r_h(s,a^*) - \brk[s]*{T_hV_{h+1}^*}(s,a^*, \tau_{1:h}) \\
    & = \bar r_h^k(s, a^*) - r_h(s,a^*) + \brk[s]*{\brk*{T_h^{
    \hat P_k, \bar c_k} - T_h }V_{h+1}^*} (\tau_{1:h+1}^*) \\
     & = \bar r_h^k(s, a^*) - r_h(s,a^*)  + \brk[s]*{\brk*{T_h^{\hat P_k, \bar c_k} - T_h^{P, \bar c_k}}V_{h+1}^*} (\tau_{1:h+1}^*) + \brk[s]*{\brk*{T_h^{P, \bar c_k} - T_h }V_{h+1}^*} (\tau_{1:h+1}^*) \\
     & \overset{(3)}= \bar r_h^k(s, a^*) - r_h(s,a^*) \\
     & \quad + \brk*{1-\rho\brk1{\sinner{d^{\tau_{h+1}^*},\bar c}}}\brk[s]*{\brk*{\hat P_h - P} V_{h+1}^*} (\tau_{1:h+1}^*) \\
     & \quad+ \brk*{\rho\brk1{\sinner{d^{\tau_{h+1}^*}, c}} - \rho\brk1{\sinner{d^{\tau_{h+1}^*},\bar c}}} \brk[s]*{P V_{h+1}^* }(\tau_{1:h+1}^*) \\
     & \overset{(4)}= \underbrace{\hat r_h^k(s, a^*) - r_h(s,a^*) + b_k^r(h,s,a^*)}_{(a)}\\
     & \quad + \underbrace{\brk*{1-\rho\brk1{\sinner{d^{\tau_{h+1}^*},\bar c}}}\brk[s]*{\brk*{\hat P_h - P} V_{h+1}^*} (\tau_{1:h+1}^*) + b_k^p(h,s,a^*)}_{(b)}   \\
     & \quad+ \underbrace{\brk*{\rho\brk1{\sinner{d^{\tau_{h+1}^*}, c}} - \rho\brk1{\sinner{d^{\tau_{h+1}^*},\hat c - b_k^c}}} \brk[s]*{P V_{h+1}^* }(\tau_{1:h+1}^*) }_{(c)}
\end{align*}
Relation $(1)$ is is due to the definition of the $\max$-operator and the fact that $a^*$ is the optimal action in the true MDP. Next, $(2)$ is due to the induction step that $\bar{V}_h^k\ge V_h^*$. $(3)$ is by the TerMDP model assumption (and since the value at termination is $0$), and $(4)$ is by the definitions of $\bar r_h^k(s,a)$ and $\bar c$.

We now turn to bound each of the three terms under the good event.
\paragraph{Term (a): Reward Optimism.}
    \begin{align*}
        (a) &= \hat r_h^k(s, a^*) - r_h(s,a^*) + b_k^r(h,s,a^*) \geq -b_k^r(h,s,a^*) + b_k^r(h,s,a^*) = 0,
    \end{align*}
    where the first transition holds under the good event, and specifically, event $\neg F^r$).
\paragraph{Term (b): Transition Optimism.}
    \begin{align*}
        (b) &= \brk*{1-\rho\brk1{\sinner{d^{\tau_{h+1}^*},\bar c}}} \brk[s]*{\brk*{\hat P_h - P} V_{h+1}^*} (\tau_{1:h+1}^*) + b_k^p(h,s,a^*) \\
        &  \geq -\abs{1-\rho\brk1{\sinner{d^{\tau_{h+1}^*},\bar c}}}\norm{\hat P_h(\cdot \mid s,a^*) - P(\cdot \mid s,a^*)}_1 \norm{ V_{h+1}^*(\cdot, \tau_{1:h+1}^*) }_\infty + b_k^p(h,s,a^*) \\
        & \geq -H \norm{\hat P_h(\cdot \mid s,a^*) - P(\cdot \mid s,a^*)}_1 + b_k^p(h,s,a^*) \\
        & \geq -b_k^p(h,s,a^*) + b_k^p(h,s,a^*) \\
        & = 0.
    \end{align*}
    The first transition is by H\"older's inequality. The second is by the fact $\forall x$, $\rho(x)\in[0,1]$ and $\norm{ V_{h+1}^*(\cdot, \tau_{1:h+1}^*) }_\infty\leq H$. The third transition is by under the good event, and specifically, event $\neg F^p$.
\paragraph{Term (c): Termination Cost Optimism.}
First, notice that under the good event (and specifically, event $\neg F^c$, see \Cref{prop: local cost confidence}) and by the definition of the bonus $b_k^c$, it holds for any $h,s,a$ that $\hat c_h(s,a) - b_k^c(h,s,a) \leq c_h(s,a)$. Therefore, as for any $h,s,a$ and $\tau$, $d^{\tau}(h,s,a)\geq 0 $, it also holds that $\inner{d^{\tau_{h+1}^*},\hat c - b_k^c} \leq \inner{d^{\tau_{h+1}^*}, c } $. Finally, by the monotonicity of $\rho$ and the non-negativity of $\brk[s]*{P V_{h+1}^* }(\tau_{1:h+1}^*) $, we have that
\begin{align*}
    (c) &= \brk*{\rho\brk1{\sinner{d^{\tau_{h+1}^*}, c}} - \rho\brk1{\sinner{d^{\tau_{h+1}^*},\hat c - b_k^c}}}  \brk[s]*{P V_{h+1}^* }(\tau_{1:h+1}^*)  \\
    & \geq \brk*{\rho\brk1{\sinner{d^{\tau_{h+1}^*}, c}} - \rho\brk1{\sinner{d^{\tau_{h+1}^*},c}}}  \brk[s]*{P V_{h+1}^* }(\tau_{1:h+1}^*) \\
    & = 0
\end{align*}

By plugging in the bounds for each of the above terms we get that
\begin{align*}
    \bar V_h^k&\brk*{s, \tau_{1:h}} - V_h^*\brk*{s, \tau_{1:h}} \geq 0,
\end{align*}
which concludes the proof by the induction hypothesis.
\end{proof}

\newpage
\section{Regret Analysis}\label{supp: regret analysis}
\label{appendix: regret analysis}

We state the main result (\Cref{theorem: main}) explicitly below. We note that in \Cref{theorem: main} we used the $\Ob$-notation, which assumes that $L = \Ob\brk*{SAH}$.

\begin{theorem}[Regret of TermCRL]
With probability at least $1-\delta$, the regret of \Cref{alg: Termination CRL} is
    \begin{align*}
        \Reg{K} 
        \leq 
        &\brk*{16H^2\log\brk*{\frac{4}{\delta}}+4SAH^2 +2\sqrt{2}\sqrt{SAH^2 K\log HK}} \\
        &\qquad 
        \times
        \left(2\sqrt{2\log \frac{8SAHK}{\delta}} + 2H\sqrt{4S\log\frac{12SAHK}{\delta}} \right.\\
         &\qquad
         +
         \left. 32H^2\sqrt{\kappa SAH}(L+1)^{1.5}\log\brk*{\frac{16}{\delta}\brk*{1+\frac{k(L+0.5)}{16S^2A^2H}}}\brk*{\sqrt{\frac{L+0.5}{4SAH}}\vee1}\right).
    \end{align*}
\end{theorem}
\begin{proof}

Under the good event, the regret is bounded by
\begin{align*}
    \Reg{K}
    &= \sum_{k=1}^K V_1^*(s_1^k) - V_1^{\pi^k}(s_1^k) \\
    &\overset{(1)}{\le} \sum_{k=1}^K \bar{V}_1^k(s_1^k) - V_1^{\pi^k}(s_1^k) \\
    & \overset{(2)}= \sum_{k=1}^K \sum_{h=1}^H \E{(\bar{r}_h^k-r)(s_h^k,a_h^k) + (\Thatpbarc -T_h)(\cdot\vert s_h^k,a_h^k,\tau_{1:h}^k)^T \bar V_{h+1}^{\pi^k}(\cdot, \tau_{1:h+1}^k) | \F_{k-1}} \\
    &  \overset{(3)}\leq \sum_{k=1}^K \sum_{h=1}^H \E{(\bar{r}_h^k-r)(s_h^k,a_h^k) \vert \F_{k-1}} + \sum_{k=1}^K \sum_{h=1}^H \E{\norm{(\Thatpbarc - T_h)(\cdot\vert s_h^k,a_h^k,\tau_{1:h}^k)}_1\norm{\bar V_{h+1}^{\pi^k}(\cdot, \tau_{1:h+1}^k)}_{\infty} | \F_{k-1}} \\
    &\le  \underbrace{\sum_{k=1}^K \sum_{h=1}^H \E{(\bar{r}_h^k-r)(s_h^k,a_h^k) | \F_{k-1}}}_{(a)} \\
    & + \underbrace{H\sum_{k=1}^K \sum_{h=1}^H \E{\norm{(\Thatpbarc - \Tpbarc)(\cdot\vert s_h^k,a_h^k,\tau_{1:h}^k)}_1 | \F_{k-1}}}_{(b)} \\
    &\quad+ \underbrace{H\sum_{k=1}^K \sum_{h=1}^H \E{\norm{(\Tpbarc - T_h)(\cdot\vert s_h^k,a_h^k,\tau_{1:h}^k)}_1 | \F_{k-1}}}_{(c)}.
\end{align*}
$(1)$ is due to the optimism of the value function (see \Cref{lemma: optimism}) and $(2)$ is by the value difference lemma (\Cref{lemma: value difference lemma}). Notice that to use this lemma, we extended the state space to include the previously visited trajectory. $(3)$ is due to Hölder's inequality.

\paragraph{Term (a): Reward Concentration.}
Under the Good event (see \Cref{lemma: good event}), and by the definition of the bonus terms $b_k^r$ we have that 
\begin{align*}
    \sum_{k=1}^K \sum_{h=1}^H& \E{(\bar{r}_h^{k-1}-r)(s_h^k,a_h^k) | \F_{k-1}}   &\le \sum_{k=1}^K \sum_{h=1}^H \E{2b_k^r(h,s_h^k,a_h^k) + b^p_k(h,s_h^k,a_h^k) | \F_{k-1}}.
\end{align*}

\paragraph{Term (b): Transition Concentration.}
Recall that
\begin{align*}
T_h^{P, c}(\cdot | s_h^k, a_h^k, \tau_{1:h}^k) = \rho\brk1{\sinner{\hat d_h^k, c}}P(\cdot | s_h^k, a_h^k),
\end{align*}
and since $\rho \in (0,1)$ we have that 
\begin{align*}
    \sum_{k=1}^K \sum_{h=1}^H& H\E{\norm{(\Thatpbarc - \Tpbarc)(\cdot\vert s_h^k,a_h^k,\tau_{1:h}^k)}_1 \vert \F_{k-1}} \\
    & \le \sum_{k=1}^K \sum_{h=1}^H H\E{\norm{(\hat P_k - P)(\cdot\vert s_h^k,a_h^k)}_1 \vert \F_{k-1}} \\
    & \le \sum_{k=1}^K \sum_{h=1}^H \E{b^p_k(h,s_h^k,a_h^k) \vert \F_{k-1}},
\end{align*}
where the last transition is by the good event (\Cref{lemma: good event}) and the definition of $b^p_k$. 

\paragraph{Term (c): Termination Cost Concentration.}

\begin{align*}
    H\sum_{k=1}^K \sum_{h=1}^H& \E{\norm{(\Tpbarc - T_h)(\cdot\vert s_h^k,a_h^k,\tau_{1:h}^k)}_1 | \F_{k-1}} \\
        & = H \sum_{k=1}^K \sum_{h=1}^H \E{\abs{\rho\brk1{\sinner{\hat d_h^k,\bar c^k}} - \rho\brk1{\sinner{\hat d_h^k, c}}} \underbrace{\norm{P_h(\cdot\vert s_h^k,a_h^k)}_1}_{=1}+\abs{\brk*{1-\rho\brk1{\sinner{\hat d_h^k,\bar c^k}}} - \brk*{1-\rho\brk1{\sinner{\hat d_h^k, c}}}} \vert \F_{k-1}} \\
        & = 2H \sum_{k=1}^K \sum_{h=1}^H \E{\abs{\rho\brk1{\sinner{\hat d_h^k,\bar c^k}} - \rho\brk1{\sinner{\hat d_h^k, c}}} | \F_{k-1}} \\ 
        & \overset{(\romannumeral 1)}\leq 2H \sum_{k=1}^K \sum_{h=1}^H \E{\abs{\inner{\hat d_h^k,\bar c^k - c}} | \F_{k-1}} \\
        & \overset{(\romannumeral 2)}\leq 2H \sum_{k=1}^K \sum_{h=1}^H \sum_{t=1}^h \E{ \abs{\bar c_t^k(s_t^k,a_t^k) - c_t(s_t^k,a_t^k)} | \F_{k-1}} \\
        & \overset{(\romannumeral 3)}\leq 4 H \sum_{k=1}^K \sum_{h=1}^H \sum_{t=1}^h \E{ b^c_k(t,s_t^k,a_t^k) | \F_{k-1}} \\
        & \leq 4 H^2  \sum_{k=1}^K \sum_{h=1}^H \E{ b^c_k(h,s_h^k,a_h^k) | \F_{k-1}},
\end{align*}
$(\romannumeral 1)$ is by the fact $\rho(\cdot)$ is $1$-Lipschitz. $(\romannumeral 2)$ is by the fact $\forall h,s,a$ $\hat d_h^k \in \{0,1\}$. $(\romannumeral 3)$ is by the definition of $\bar c$ and under the good event by \Cref{prop: local cost confidence}.

Combining all the above terms, we get
\begin{align*}
    \Reg{K} &\leq 4 \sum_{k=1}^K \sum_{h=1}^H \E{b_k^r(h,s_h^k,a_h^k) + b^p_k(h,s_h^k,a_h^k) + H^2 b^c_k(h,s_h^k,a_h^k)  \vert \F_{k-1}}.
\end{align*}
Now, plugging in the bonus terms and bounding, 
\begin{align*}
    b_k^c(h,s,a) 
    &= 24\sqrt{\kappa SAH^{2.5}}(L+1)^{1.5}\log^2\brk*{\frac{16}{\delta}\brk*{1+\frac{k(L+0.5)}{16S^2A^2\sqrt{H}}}}\frac{1}{\sqrt{n^k_h(s,a) + 4\frac{SAH}{L\sqrt{H}+0.5}}}\\
    & \le 24\sqrt{\kappa SAH^{2.5}}(L+1)^{1.5}\log^2\brk*{\frac{16}{\delta}\brk*{1+\frac{k(L+0.5)}{16S^2A^2\sqrt{H}}}}\frac{\brk*{\sqrt{\frac{L+0.5}{4SA\sqrt{H}}}\vee1}}{\sqrt{n^k_h(s,a) \vee 1}} \enspace,
\end{align*}
we get
{\small
\begin{align*}
    \Reg{K} & \leq 2 \sum_{k=1}^K \sum_{h=1}^H \E{\sqrt{\frac{2\log \frac{8SAHK}{\delta}}{n_h^k(s_k,a_k)\vee 1}} +H\sqrt{\frac{4S\log\frac{12SAHK}{\delta}}{n_h^k(s_k,a_k)\vee 1}} | \F_{k-1}} \\
    &\quad  +4 H^2\sum_{k=1}^K \sum_{h=1}^H \E{24\sqrt{\kappa SAH^{2.5}}(L+1)^{1.5}\log^2\brk*{\frac{16}{\delta}\brk*{1+\frac{k(L+0.5)}{16S^2A^2\sqrt{H}}}}\frac{\brk*{\sqrt{\frac{L+0.5}{4SA\sqrt{H}}}\vee1}}{\sqrt{n^k_h(s,a) \vee 1}}| \F_{k-1}} \\
    & \le\brk*{16H^2\log\brk*{\frac{4}{\delta}}+4SAH^2 +2\sqrt{2}\sqrt{SAH^2 K\log HK}} \times \\
    &\qquad \left(2\sqrt{2\log \frac{8SAHK}{\delta}} + 2H\sqrt{4S\log\frac{12SAHK}{\delta}} \right.\\
                 &\qquad\qquad\left.+ 92H^2\sqrt{\kappa SAH^{2.5}}(L+1)^{1.5}\log^2\brk*{\frac{16}{\delta}\brk*{1+\frac{k(L+0.5)}{16S^2A^2\sqrt{H}}}}\brk*{\sqrt{\frac{L+0.5}{4SA\sqrt{H}}}\vee1}\right) \\
                 & = \Ob\brk*{\sqrt{\kappa S^2A^2H^{8.5}L^3 K\log^3\brk*{ \frac{SAHK}{\delta}}}},
\end{align*}
}
where the last inequality is by \Cref{lemma: expected cumulative visitation bound}.
\end{proof}

\newpage
\section{Cost Concentration}
\label{supp: cost concentration}

In the section to follow, we provide a local concentration bound for estimating the costs in the termination model. This local concentration result will allow us to use the cost bonus defined in \Cref{supp: good event}, which is crucial for achieving tight regret guarantees, when the costs are unknown. Specifically, this concentration result serves as the base for dealing with the cost optimism in \Cref{lemma: optimism} and cost concentration in \Cref{theorem: main}.

We first start by describing the optimization procedure used for learning the costs (\Cref{supp: optimization problem}). Then, by formulating this procedure as an instance of the logistic bandits problem, we obtain a concentration result for the cost, which \emph{globally bounds the distance between the empirical and cost vectors}
(\Cref{supp: global cost concentration}).
Finally, we use this global concentration result together with the structure of the termination MDP to provide a local concentration bound for the costs which \emph{holds independently for any $h,s,a$}
(\Cref{supp: local cost concentration}).

Before diving into the details, we first present additional notations required to the analysis in this appendix. 
First, to address time steps after termination, we define $\hat d_h^k(t,s,a)=0$ for any $t>t_k^*$ and $\forall t,s,a$. This allows to add these vectors to the regression problem without affecting its solution. Moreover, we denote the episode-wise Gram matrix by $A_k \triangleq \sum_{h=1}^{t^*} \hat d_h^k  {\hat d^k_h \vphantom{}}^{T}$ and the regularized Gram matrix up to episode $k$ by $ V_k \triangleq \lambda I + \sum_{k'=1}^{k-1} A_{k'} $, for some $\lambda>0$.

\subsection{Optimization Procedure}\label{supp: optimization problem}

In a vanilla MDP, at each time step, the agent takes an action and transitions to the next state. Instead, in the termination model considered in this paper, the agent potentially gets a termination signal from the environment. Importantly, the agent can gain information about the termination probabilities even in time steps when no termination occurs. Formally, at any time step up to the time of termination $h\in[H-1]$ (the last time step always ends in termination), the agent acquires a sample from a Bernoulli random variable, where $1$ means termination $(h=t_k^*)$, and $0$  otherwise $(h\neq t_k^*)$. Since the termination probability is a logistic function of the occupancy vector, finding the termination costs can be done using a logistic regression. This leads us to solve the following optimization problem in \Cref{alg: Termination CRL} (Line 10),

\begin{align}\label{eq: logistic cost estimation process}
   \hat c_{k+1} \in \arg\max_{c \in \C} \sum_{k=1}^{K} \sum_{h=1}^{H-1} \brk[s]*{ \indicator{h < t_k^*}\log\brk*{1- \rho\brk1{\sinner{\hat d_h^k,c}}} + \indicator{h=t_k^*}\log\brk*{\rho\brk*{\sinner{\hat d_h^k,c}}
    }}  -  \lambda\norm{c}_2^2,
\end{align}
where $\C$ is the set of possible costs. Specifically, by the fact we only observe termination on the trajectories we played, the feedback in this problem is bandit feedback. In the next section, we use known results for the logistic bandit problem \cite{abeille2021instance}, and apply them to acquire a global concentration bound for the costs learned in this procedure.

\subsection{Global Cost Concentration}\label{supp: global cost concentration}

We start by stating concentration results for logistic bandits. Let $x_t\in\X\subset\R^d$ be a sequence of contexts s.t. $\norm{x}\le 1$ for all $x\in\X$ and assume that $y_t$ are Bernoulli random variable such that $\E{y_t\vert x_t}=\rho(x_{t}^T\theta_*)$, for some $\theta^*\in\Theta$. Next, define the regularized logistic loss w.r.t. $\lambda>0$ by $\mathcal{L}_t(\theta) = - \sum_{s=1}^{t-1}\brk*{y_s\log \rho(x_{s}^T\theta)+(1-y_s)\log(1- \rho(x_{s}^T\theta))} + \lambda\norm{\theta}_2^2$. We denote the unconstrained minimizer of the loss by $\bar{\theta}_t$ and its minimizer over $\Theta$ by $\hat{\theta}_t$. Finally, with slight abuse of notation, let $H_t(\theta)=\sum_{s=1}^{t-1}\dot{\mu}(x_s^T\theta)x_sx_s^T + \lambda I$. 

We now focus on the confidence set $\mathcal{E}_t\delta)$, defined in \citep{abeille2021instance} as
\begin{align*}
    \mathcal{E}_t(\delta) = \brk[c]*{\theta\in\Theta \vert \mathcal{L}_t(\theta) - \mathcal{L}_t(\bar{\theta}_t)\le \beta_t(\delta)^2}
\end{align*}
for some appropriate $\beta_t(\delta)>0$. By Proposition 1 and Lemma 1 of \citep{abeille2021instance}, it holds that $\theta^*\in \mathcal{E}_t(\delta)$. Thus, under this event, the set is not empty. In particular, when the set is not empty, it must also contains $\hat{\theta}_t$, as the constrained minimizer of the cost. Combining this conclusion with Lemma 1 of \citep{abeille2021instance}, we obtain the following lemma.

\begin{lemma}[Logistic Regression Concentration]\label{lemma: logistic concentration}
With probability of at least $1-\delta$, for all $t\ge1$,
\begin{align}
    \label{eq:theta-bound}
    \norm{\hat{\theta}_t-\theta_*}_{H_t(\theta_*)}
    \le (2+2L)\gamma_t(\delta) + 2\sqrt{1+L}\brk*{\gamma_t(\delta) + \frac{\gamma_t(\delta)^2}{\sqrt{\lambda}}}\enspace,
\end{align}
where $L=\max_{\theta\in\Theta}\norm{\theta}_2$ and $\gamma_t(\delta) = \sqrt{\lambda}\brk*{L+\frac{1}{2}}+\frac{d}{\sqrt{\lambda}}\log\brk*{\frac{4}{\delta}\brk*{1+\frac{t}{16d\lambda}}}$.
\end{lemma}

This result can be applied to our algorithm as follows:
\begin{corollary}[Global Cost Bound]\label{corollary: global cost bound}
Let $\lambda=\frac{d}{L\sqrt{H}+0.5}$, ${\X=\brk[c]*{x\in\R^{SAH}: x_i\in\brk[c]{0,1},\sum_i x_i\le H}}$ and ${\kappa=
\brk*{\min_{x\in\X}\dot{\mu}(x^Tc)}^{-1}}$. With probability of at least $1-\delta$, it holds that 
\begin{align*}
    \norm{\hat{c}^k-c}_{V_k}\le 12\sqrt{\kappa SAH^{2.5}}(L+1)^{1.5}\log^2\brk*{\frac{4}{\delta}\brk*{1+\frac{k(L+0.5)}{16S^2A^2\sqrt{H}}}}.
\end{align*}
\end{corollary}
\begin{proof}
We apply \Cref{lemma: logistic concentration} with $x_{H(k-1)+h}=\hat{d}_h^k$ if $h\le \min\brk*{t^*_k,H-1}$ and $x_{H(k-1)+h}=0$ otherwise, where $k>0$ is the episode and $h\in\brk[s]*{H}$ is the step within the episode. Notice that zero contexts $x_{H(k-1)+h}=0$ do not affect the regression solution, so steps after termination are ignored. Also, in our case, we have $\norm{x}^2\le H$, instead of $\norm{x}\le1$, so to compensate it, the parameter $L$ will be scaled by $\sqrt{H}$. Under our assumptions, we have $d=SAH$, and thus 
$\gamma_k(\delta)\le 2\sqrt{SAH(L\sqrt{H}+0.5)}\log\brk*{\frac{4}{\delta}\brk*{1+\frac{k(L\sqrt{H}+0.5)}{16S^2A^2H}}}$ and
\begin{align*}
    \norm{\hat{c}^k-c}_{H_t(c)}
    &\le (2+2L\sqrt{H})\gamma_t(\delta) + 2\sqrt{1+L\sqrt{H}}\brk*{\gamma_t(\delta) + \frac{\gamma_t(\delta)^2}{\sqrt{\lambda}}} \\
    &\le 4(1+L\sqrt{H})\gamma_t(\delta) + \frac{\sqrt{1+L\sqrt{H}}}{\sqrt{\lambda}}\gamma_t(\delta)^2 \\
    &\le 12\sqrt{SAH}(L\sqrt{H}+1)^{1.5}\log^2\brk*{\frac{4}{\delta}\brk*{1+\frac{k(L\sqrt{H}+0.5)}{16S^2A^2H}}} \\
    & \le 12\sqrt{SAH^{2.5}}(L+1)^{1.5}\log^2\brk*{\frac{4}{\delta}\brk*{1+\frac{k(L+0.5)}{16S^2A^2\sqrt{H}}}}
\end{align*}
We conclude the proof by using the fact that $\kappa H_k(c) \succeq V_t$ (\Cref{lemma: local to global design matrix}) for $\theta^* = c$.
\end{proof}

\begin{lemma}[Connection between Local and Global Design Matrices]\label{lemma: local to global design matrix}
For any $\theta$, if $\kappa_\X(\theta)=\frac{1}{\min_{x\in\X}\dot{\mu}(x^T\theta)}$, then it holds that
\begin{align*}
    \kappa_{\X}(\theta) H_t(\theta) \succeq V_t\enspace.
\end{align*}
\end{lemma}
\begin{proof}
For any $\theta$,
\begin{align*}
    H_t(\theta) & = \sum_{s=1}^{t-1}\dot{\mu}(x_s^T\theta)x_sx_s^T + \lambda_t I \\
    &\succeq \min_{x\in\X}\dot{\mu}(x_s^T\theta)\sum_{s=1}^{t-1}x_sx_s^T + \lambda_t I\\
    &= \frac{1}{\kappa_{\X}(\theta)}\sum_{s=1}^{t-1}x_sx_s^T + \lambda_t I\\
    &\overset{1< \kappa_{\X}(\theta)}{\succeq} \frac{1}{\kappa_{\X}(\theta)}\brk*{\sum_{s=1}^{t-1}x_sx_s^T + \lambda_t I}
    = \frac{1}{\kappa_{\X}(\theta)}V_t\enspace.
\end{align*}
\end{proof}

\subsection{From Global to Local Cost Concentration}\label{supp: local cost concentration}

In our setting, the global cost concentration also implies a component-wise (local) concentration of the costs. This is in stark contrast to standard regression results, where specific coordinates cannot always be recovered. To see the intuition behind it, assume for simplicity that we observe the exact termination probability (rather than just a random termination signal). Notably, on each time step up to the termination, we could directly reconstruct the partial sums $\sum_{t=1}^h c_t(s_t^k,a_t^k)$ by inverting the logistic function. Then, the costs of all visited states could be directly calculated through the difference between the cumulative costs of any two consecutive steps. Notice that we can calculate this difference only since we get feedback on \emph{every} time step up to termination. Were we only to observe feedback on the cumulative costs at the termination time, then we would not be able to guarantee to reconstruct anything but the global cumulative costs we were given (see example in \Cref{fig: globality of last-iterate feedback}). In this case, the global (vector-wise) concentration used in logistic bandits (see \Cref{corollary: global cost bound} in \Cref{supp: global cost concentration}) cannot not be improved.

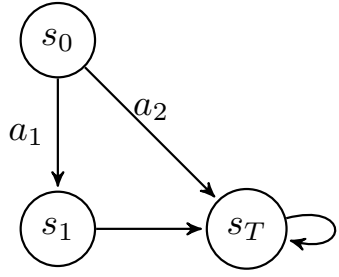
\begin{figure}
	\centering
	\resizebox{2in}{!}{
		\begin{tikzpicture}[->,>=stealth',shorten >=1pt,auto,node distance=1.8cm,
		semithick, state/.style={circle, draw, minimum size=0.3cm}]
		\tikzstyle{every state}=[thick]
		]
		
		\node[state] (S0) {$s_0$};
		\node[state] (S1) [below of=S0] {$s_1$};
		\node[state] (S2) [right of=S1] {$s_T$};
	
		\path
		(S0) edge  [below]  node[pos=0.1,above ]{ }  node [left] {$a_1$} (S1)
		(S0)	edge  [right]   node[pos=0.1,below ]{ }         node [above] {$a_2$} (S2) 
		(S1) edge [right] node[pos=1,right]{} node {} (S2) 
	    (S2) edge  [loop right] node[pos=0.1,above]{}  node {} (S2);
		\end{tikzpicture}
	}
	\caption{\small Example for the necessity of step-wise feedback for local estimation. Plot depicts a deterministic MDP which always begins at state $s_0$. Assume that we only receive a binary signal of whether there was termination at the end of the episode, but do not observe it during the episode. If we only observe the trajectory $s_0 \to s_1 \to s_T$, identifying the state in which the termination occurred is not possible, until we further observe the trajectory $s_0 \to s_T$. That is, there is no way to know whether termination occurred due to state $s_0$ or $s_1$. In contrast, in the TerMDP model we view termination signals at every iteration, allowing us to identify the costs locally w.r.t. every state and action. }
	\label{fig: globality of last-iterate feedback}
\end{figure}

\begin{lemma}[Local Cost Estimation Confidence Bound]
\label{prop: local cost confidence}
For any $\delta>0$, with probability of at least $1-\delta$, for any episode $k\in [K]$, it holds that for all $h\in[H-1],s\in\sset ,a\in\aset$,
\begin{align*}
    \abs{\hat c_h^k(s,a) - c_h(s,a) } \leq  24\sqrt{\kappa SAH^{2.5}}(L+1)^{1.5}\log\brk*{\frac{4}{\delta}\brk*{1+\frac{k(L+0.5)}{16S^2A^2\sqrt{H}}}}\frac{1}{\sqrt{n^k_h(s,a) + 4\frac{SAH}{L\sqrt{H}+0.5}}}.
\end{align*}
\end{lemma}
\begin{proof}
    For any $k\in\brk[s]{K}$, $h\in\brk[s]{H-1}$, $s\in\sset$ and  $h\in\sset$, we can bound
    \begin{align}\label{eq: single hsa concentration part 1}
        \abs{c_h^k(s,a)-c_h(s,a)} = \abs{\brk[a]*{e_{h,s,a}, \hat c_k - c}} \leq \norm{e_{h,s,a}}_{V_k^{-1}} \norm{\hat c_k - c}_{V_k},
    \end{align}
    where $e_{h,s,a}\in\R^{HSA}$ is a unit vector in the $(h,s,a)$ coordinate, and the inequality is due to Cauchy-Schwartz.
    We now turn to bound $\norm{e_{h,s,a}}_{V_k^{-1}}$.
    
    \begin{align*}
        V_k^{-1} & 
        \preceq \brk*{\lambda I + \sum_{k'=1}^{k-1} A_{k'} \mathbbm{1} \brk[c]*{(h,s,a)\in \tau^{k'}}}^{-1}
        \\ & =
        \brk*{\sum_{k'=1}^{k-1}  \brk*{\frac{\lambda}{n^{k-1}_h(s,a)} I + A_{k'}}\mathbbm{1} \brk[c]*{(h,s,a)\in \tau^{k'}}}^{-1}
        \\ & \preceq \frac{1}{\brk*{n^{k-1}_h(s,a)}^2}
        \sum_{k'=1}^{k-1}  \brk*{\frac{\lambda}{n^{k-1}_h(s,a)} I + A_{k'}}^{-1}\mathbbm{1} \brk[c]*{(h,s,a)\in \tau^{k'}}\enspace,
    \end{align*}
    where $n^{k-1}_h(s,a) = \sum_{k'=1}^{k-1} \mathbbm{1} \brk[c]*{(h,s,a)\in \tau^{k'}}$ and the third transition is due to HM-AM inequality for positive matrices \cite{bhagwat1978inequalities}.
    
    Using this relation, we can bound $\norm{e_{h,s,a}}_{V_k^{-1}}^2$ by bounding the maximal eigenvalue of each of the summands. We do so in \Cref{lemma: Inverse Eigenvalues bound}, and obtain the following bound:
    \begin{align*}
         \norm{e_{h,s,a}}_{V_k^{-1}}^2 &= e_{h,s,a}^T V_k^{-1} e_{h,s,a} 
         \\ &
         \leq    \frac{1}{\brk*{n^{k-1}_h(s,a)}^2}
        \sum_{k'=1}^{k-1}  e_{h,s,a}^T\brk*{\frac{\lambda}{n^{k-1}_h(s,a)} I + A_{k'}}^{-1}\mathbbm{1} \brk[c]*{(h,s,a)\in \tau^{k'}}  e_{h,s,a} \\
        & \leq 
        \frac{1}{\brk*{n^{k-1}_h(s,a)}^2}
        \sum_{k'=1}^{k-1}  \frac{1}{\frac{1}{4} + \frac{\lambda}{n^{k-1}_h(s,a)}}\mathbbm{1} \brk[c]*{(h,s,a)\in \tau^{k'}} \tag{\Cref{lemma: Inverse Eigenvalues bound}}\\
        & = \frac{n^{k-1}_h(s,a)}{\brk*{n^{k-1}_h(s,a)}^2}
        \frac{1}{\frac{1}{4} + \frac{\lambda}{n^{k-1}_h(s,a)}} \\
        & = \frac{4}{n^{k-1}_h(s,a) + 4\lambda}.
    \end{align*}

    By plugging into \cref{eq: single hsa concentration part 1}, we obtain that for any $k$ and any $h,s,a$
    \begin{align*}
        \abs{c_h^k(s,a)-c_h(s,a)}
        \leq  \norm{\hat c_k - c}_{V_k}\norm{e_{h,s,a}}_{V_k^{-1}} \leq  \norm{\hat c_k - c}_{V_k}\sqrt{\frac{4}{n^{k-1}_h(s,a) + 4\lambda}}.
    \end{align*}
    Finally, by \Cref{corollary: global cost bound} with $d=SAH$, with probability of at least $1-\delta$ it holds that
        \begin{align*}
        \abs{c_h^k(s,a)-c_h(s,a)}
        & 24\sqrt{\kappa SAH^{2.5}}(L+1)^{1.5}\log\brk*{\frac{4}{\delta}\brk*{1+\frac{k(L+0.5)}{16S^2A^2\sqrt{H}}}}\frac{1}{\sqrt{n^k_h(s,a) + 4\frac{SAH}{L\sqrt{H}+0.5}}},
    \end{align*}
    
    which concludes the proof.

\end{proof}

\begin{lemma}[Inverse Eigenvalues Bound]\label{lemma: Inverse Eigenvalues bound}
    If $(h,s,a)\in \tau^{k}$ and $e_{h,s,a}\in\R^{HSA}$ is a unit vector in the coordinate $(h,s,a)$, then $e_{h,s,a}^T\brk*{\lambda I + A_{k}}^{-1}e_{h,s,a}\le \frac{1}{\frac{1}{4}+\lambda}$ .
\end{lemma}
\begin{proof}
Throughout the proof, we assume for brevity that there was termination in all episodes, i.e., $t_k^*\le H-1$ for all $k\in\brk[s]{K}$. Otherwise, the exact same proof follows by replacing $t_k^*$ by $\min\brk[c]*{t_k^*,H-1}$ (namely, treating the lack of termination feedback at the last timestep of the episodes). With some abuse of notations, we also use $e_i\in\R^{HSA}$ to denote the unit vector in the $i$-th coordinate.

To simplify the proof, it would be helpful to assume that the $t$-th coordinate of the empirical occupancy vector represents the state that was visited on the $t$-th time step (namely, coordinates are sorted by their visitation order). States that were not visited can be arbitrarily ordered. Formally, this can be done using any permutation matrix $P_k$ such that $e_{t,s_t^k,a_t^k}=P_k e_t$ for all $t\in\brk[s]{t_k^*}$ (and other coordinates can be arbitrarily permuted). In particular, denoting $\bar{e}_t = \sum_{i=1}^t e_{i}=\brk1{\underbrace{1,\dots,1}_{t-\mathrm{times}},0,\dots,0}^T$, we can write $\hat d_t^k = \sum_{i=1}^te_{i,s_i^k,a_i^k} = \sum_{i=1}^t P_ke_i = P_k\bar{e}_t$. 

Then, we have that
\begin{align*}
    e_{h,s,a}^T\brk*{\lambda I + A_{k}}^{-1}e_{h,s,a}
    &= e_{h,s,a}^T\brk*{\lambda I + \sum_{t=1}^{t^*_k}\hat d_t^k{\hat d^k_t \vphantom{ }}^{T}}^{-1}e_{h,s,a} \\
    &= e_{h,s,a}^T\brk*{\lambda I + \sum_{t=1}^{t^*_k}P_k\bar e_t \bar e_t^T P_k^T}^{-1}e_{h,s,a} \\
    &= e_{h,s,a}^T\brk*{P_k\brk*{\lambda I + \sum_{t=1}^{t^*_k}\bar e_t \bar e_t^T} P_k^T}^{-1}e_{h,s,a} \\
    &= e_{h,s,a}^TP_k\brk*{\lambda I + \sum_{t=1}^{t^*_k}\bar e_t \bar e_t^T}^{-1}P_k^Te_{h,s,a}\enspace,
\end{align*}
where the two last relations is since permutation matrices are orthogonal, namely $P_k^{-1} = P_k^T$. 

Now, notice that $P_k$ permutes the first $t_k^*$ components to the visited $(t,s_t^k,a_t^k)$ tuples. Thus, its inverse $P_k^T$ permutes visited tuples $(t,s_t^k,a_t^k)$ to the $t$-th coordinate, namely $P_k^Te_{t,s_t^k,a_t^k}=e_t$ and 
\begin{align*}
    e_{h,s,a}^T\brk*{\lambda I + A_{k}}^{-1}e_{h,s,a} = e_{h}^T\brk*{\lambda I + \sum_{t=1}^{t^*_k}\bar e_t \bar e_t^T}^{-1}e_{h}\enspace.
\end{align*}

Moreover, $\lambda I + \sum_{t=1}^{t^*_k}\bar e_t \bar e_t^T$ is a block-diagonal matrix whose first block is located at its first $t^*_k$ coordinates. Thus, each block can be inverted independently, and as $e_h$ is located in the first block of the matrix, w.l.o.g. we can only focus on the $t^*_k\times t^*_k$ first-block of the matrix. We denote this block by $B\in\R^{t^*_k\times t^*_k}$, and if $u_h\in\R^{t_k^*}$ is a unit vector in the $h^{\text{th}}$ coordinate, we can write $e_{h}^T\brk*{\lambda I + \sum_{t=1}^{t^*_k}\bar e_t \bar e_t^T}^{-1}e_{h} = u_h^T(\lambda I+B)^{-1} u_h$.

Directly calculating of the sum $\sum_{t=1}^{t^*_k}\bar e_t \bar e_t^T$, one can easily see that $B(i,j) = t^*_k + 1 -  \max\brk[c]{i,j}$, as we now illustrate:
\begin{align*}
    \brk*{
    \begin{array}{ccccc}
        t_k^* & t_k^*-1    & t_k^*-2 & \dots & 1 \\
        t_k^* -1 & t_k^*-1 & t_k^*-2 & \dots & 1 \\
        t_k^* -2 & t_k^*-2 & t_k^*-2 & \dots & 1\\
        \vdots & \vdots & \vdots & \vdots & \vdots \\
        1 & 1 & 1 & 1 & 1
    \end{array}}
\end{align*}
This matrix can be easily diagonalized, which can be used to calculate its inverse:
\begin{align*}
    &\brk*{
    \begin{array}{ccccc|ccccc}
        t_k^* & t_k^*-1    & t_k^*-2 & \dots & 1 & 1 & 0 & 0 & \dots & 0\\
        t_k^* -1 & t_k^*-1 & t_k^*-2 & \dots & 1 & 0 & 1 & 0 & \dots & 0\\
        t_k^* -2 & t_k^*-2 & t_k^*-2 & \dots & 1 & 0 & 0 & 1 & \dots & 0\\
        \vdots & \vdots & \vdots & \vdots & \vdots & \vdots & \vdots & \vdots & \vdots & \vdots\\
        1 & 1 & 1 & 1 & 1 & 0 & 0 & 0 & \dots & 1
    \end{array}}\\
    &=\brk*{
    \begin{array}{ccccc|ccccc}
        1 & 0    & 0 & \dots & 0 & 1 & -1 & 0 & \dots & 0\\
        1 & 1 & 0 & \dots & 0 & 0 & 1 & -1 & \dots & 0\\
        1 & 1 & 1 & \dots & 0 & 0 & 0 & 1 & \dots & 0\\
        \vdots & \vdots & \vdots & \vdots & \vdots & \vdots & \vdots & \vdots & \vdots & \vdots\\
        1 & 1 & 1 & 1 & 1 & 0 & 0 & 0 & \dots & 1
    \end{array}} \\
    &=\brk*{
    \begin{array}{ccccc|ccccc}
        1 & 0    & 0 & \dots & 0 & 1 & -1 & 0 & \dots & 0\\
        0 & 1 & 0 & \dots & 0 & -1 & 2 & -1 & \dots & 0\\
        0 & 0 & 1 & \dots & 0 & 0 & -1 & 2 & \dots & 0\\
        \vdots & \vdots & \vdots & \vdots & \vdots & \vdots & \vdots & \vdots & \vdots & \vdots\\
        0 & 0 & 0 & 0 & 1 & 0 & 0 & 0 & \dots & 2
    \end{array}}
\end{align*}
In the first step, we subtracted rows $i+1$ from the $i$ rows, and in the second step, we subtracted the $i-1$ rows from the $i$ rows. Thus, the inverse can be explicitly written as follows:
\begin{align*}
    B^{-1}_{i,j} = 
    \left\{
    \begin{array}{ll}
         1 & i=j=1 \\
         2 & i=j>1 \\
         -1 & i=j-1 \; \text{or} \; i=j+1 \\
         0 & \text{o.w.}
    \end{array}
    \right.
\end{align*}
Notice that the absolute values of all rows is smaller than $4$. Then (e.g., by Gershgorin circle theorem), $\lambda_{\max}(B^{-1})\le 4$, and since $B$ is PSD, $\lambda_{\min}(B)\ge\frac{1}{4}$. Finally, we get the desired result by bounding
\begin{align*}
    e_{h,s,a}^T\brk*{\lambda I + A_{k}}^{-1}e_{h,s,a} 
    &= 
    u_{h}^T\brk*{\lambda I + B}^{-1}u_{h} \\
    &\le 
    \underbrace{\norm{u_{h}}_2}_{=1}\lambda_{\max}\brk*{\brk*{\lambda I + B}^{-1}} \\
    &= 
    \frac{1}{\lambda_{\min}\brk*{\lambda I + B}} 
    \le 
    \frac{1}{\frac{1}{4}+\lambda}\enspace.
\end{align*}

\textbf{Remark.} Notice that the same conclusion still holds in the extreme case of $t^*\le2$: we get $\lambda_{\max}(B^{-1})\le 3$ for $t^*=2$ and $\lambda_{\max}(B^{-1}) = 1$ (as $B$ contains a single element).
\end{proof}


\newpage
\section{Useful Lemmas}
\label{appendix: userful lemmas}

\begin{lemma}[Value difference lemma, e.g., \citet{dann2017unifying}, Lemma E.15]\label{lemma: value difference lemma}
    Consider two MDPs $\M=\brk*{\sset,\aset,P,r,H}$ and $\M'=\brk*{\sset,\aset,P',r',H}$. For any policy $\pi$ and any $s,h$, the following relation holds:
    \begin{align*}
        &V_h^{\pi}(s;\M) - V_h^{\pi}(s;\M') \\
            & \quad = \E{\sum_{h'=h}^H \brk*{r_{h'}(s_{h'},a_{h'}) - r_{h'}'(s_{h'},a_{h'})}  + \brk*{P-P'}(\cdot \mid s_{h'},a_{h'})^T V_{h'+1}^{\pi}\brk*{\cdot ; \M'} \vert s_h = s,\pi,P}
    \end{align*}
\end{lemma}

The following lemma is due to \citet{efroni2020reinforcement}, with the only exception that the original lemma assumes a stationary MDP and therefore, $S$ translates to $SH$ in the following.
\begin{lemma}[Expected Cumulative Visitation Bound, Lemma 22, \citet{efroni2020reinforcement}] \label{lemma: expected cumulative visitation bound}
    Let $\brk[c]*{\F_{k}}_{k=1}^K$ be the natural filtration. Then, with probability greater than $1-\delta$ it holds that
    \begin{align*}
        \sum_{k=1}^K \E{\sum_{h=1}^H \frac{1}{\sqrt{n_h^k(s^k_h,a^k_h)\vee 1}} \mid \F_{k-1}} &\leq 16H^2\log\brk*{\frac{1}{\delta}}+4SAH^2 +2\sqrt{2}\sqrt{SAH^2 K\log HK}\\
        &=\Ob\brk*{H\brk*{HSA + H\log\brk*{\frac{1}{\delta}}} + \sqrt{SAH^2K\log HK}} \\
        &= \tilde \Ob \brk*{\sqrt{H^2SAK}} 
    \end{align*}
\end{lemma}

\end{document}

%% file: checklist.tex
\section*{Checklist}


\begin{enumerate}

\item For all authors...
\begin{enumerate}
  \item Do the main claims made in the abstract and introduction accurately reflect the paper's contributions and scope?
    \answerYes{}{}
  \item Did you describe the limitations of your work?
    \answerYes{}
  \item Did you discuss any potential negative societal impacts of your work?
    \answerYes{}
  \item Have you read the ethics review guidelines and ensured that your paper conforms to them?
    \answerYes{}
\end{enumerate}

\item If you are including theoretical results...
\begin{enumerate}
  \item Did you state the full set of assumptions of all theoretical results?
    \answerYes{}
        \item Did you include complete proofs of all theoretical results?
    \answerYes{}
\end{enumerate}

\item If you ran experiments...
\begin{enumerate}
  \item Did you include the code, data, and instructions needed to reproduce the main experimental results (either in the supplemental material or as a URL)?
    \answerYes{}
  \item Did you specify all the training details (e.g., data splits, hyperparameters, how they were chosen)?
    \answerYes{}
        \item Did you report error bars (e.g., with respect to the random seed after running experiments multiple times)?
    \answerYes{}
        \item Did you include the total amount of compute and the type of resources used (e.g., type of GPUs, internal cluster, or cloud provider)?
    \answerYes{}
\end{enumerate}

\item If you are using existing assets (e.g., code, data, models) or curating/releasing new assets...
\begin{enumerate}
  \item If your work uses existing assets, did you cite the creators?
    \answerYes{}
  \item Did you mention the license of the assets?
    \answerNA{}
  \item Did you include any new assets either in the supplemental material or as a URL?
    \answerYes{}
  \item Did you discuss whether and how consent was obtained from people whose data you're using/curating?
    \answerNA{}
  \item Did you discuss whether the data you are using/curating contains personally identifiable information or offensive content?
    \answerNA{}
\end{enumerate}

\item If you used crowdsourcing or conducted research with human subjects...
\begin{enumerate}
  \item Did you include the full text of instructions given to participants and screenshots, if applicable?
    \answerNA{}
  \item Did you describe any potential participant risks, with links to Institutional Review Board (IRB) approvals, if applicable?
    \answerNA{}
  \item Did you include the estimated hourly wage paid to participants and the total amount spent on participant compensation?
    \answerNA{}
\end{enumerate}

\end{enumerate}